\documentclass{article} 

\usepackage{url}
\usepackage{enumitem}
\usepackage{extarrows}
\usepackage{wrapfig}
\usepackage{caption}
\usepackage{tikz}
\usepackage{microtype}
\usepackage{colortbl,booktabs}
\usepackage{bbm}

\usepackage{soul}

\graphicspath{{./figs/}}

\usepackage{amsmath,amsfonts,bm}

\def\1{\bm{1}}

\def\va{{\bm{a}}}
\def\vb{{\bm{b}}}

\def\vd{{\bm{d}}}

\def\vg{{\bm{g}}}
\def\vh{{\bm{h}}}

\def\vm{{\bm{m}}}

\def\vo{{\bm{o}}}
\def\vp{{\bm{p}}}

\def\vs{{\bm{s}}}

\def\vx{{\bm{x}}}
\def\vy{{\bm{y}}}
\def\vz{{\bm{z}}}

\DeclareMathAlphabet{\mathsfit}{\encodingdefault}{\sfdefault}{m}{sl}
\SetMathAlphabet{\mathsfit}{bold}{\encodingdefault}{\sfdefault}{bx}{n}

\DeclareMathOperator*{\argmin}{arg\,min}

\DeclareMathOperator{\sign}{sign}

\usepackage{amssymb}
\usepackage{amsbsy}
\usepackage{amsmath}
\usepackage{amsthm}
\usepackage{graphicx}
\usepackage{color}
\usepackage{multirow}
\usepackage{booktabs}
\usepackage{wasysym}
\usepackage{mathtools} 
\usepackage{MnSymbol}

\newcommand{\vA}{{\mathbf{A}}}
\newcommand{\vB}{{\mathbf{B}}}

\newcommand{\vI}{{\mathbf{I}}}
\newcommand{\vJ}{{\mathbf{J}}}

\newcommand{\vM}{{\mathbf{M}}}

\newcommand{\vP}{{\mathbf{P}}}

\newcommand{\vU}{{\mathbf{U}}}
\newcommand{\vV}{{\mathbf{V}}}

\newcommand{\cA}{{\mathcal{A}}}

\newcommand{\cD}{{\mathcal{D}}}

\newcommand{\cF}{{\mathcal{F}}}

\newcommand{\cL}{{\mathcal{L}}}

\newcommand{\cX}{{\mathcal{X}}}

\newcommand{\cZ}{{\mathcal{Z}}}

\newcommand{\prox}{{\mathrm{prox}}}

\newtheorem{theorem}{Theorem}
\newtheorem{definition}{Definition}

\newtheorem{lemma}{Lemma}
\newtheorem{proposition}{Proposition}

\usepackage{hyperref}

\usepackage[accepted]{icml2023}

\usepackage[textsize=tiny]{todonotes}

\newcommand{\cut}[1]{{}}  

\definecolor{Xiaohan_color}{rgb}{0.098, 0.643, 0.071}

\icmltitlerunning{Towards Constituting Mathematical Structures for Learning to Optimize}

\begin{document}

\twocolumn[
\icmltitle{Towards Constituting Mathematical Structures for Learning to Optimize}


\icmlsetsymbol{equal}{*}

\begin{icmlauthorlist}
\icmlauthor{Jialin Liu}{equal,alibaba}
\icmlauthor{Xiaohan Chen}{equal,alibaba}
\icmlauthor{Zhangyang Wang}{ut}
\icmlauthor{Wotao Yin}{alibaba}
\icmlauthor{HanQin Cai}{ucf}
\end{icmlauthorlist}

\icmlaffiliation{alibaba}{Alibaba Group (U.S.) Inc, Bellevue, WA, USA}
\icmlaffiliation{ut}{Department of Electrical and Computer Engineering, University of Texas at Austin, Austin, TX, USA}
\icmlaffiliation{ucf}{Department of Statistics and Data Science and Department of Computer Science, University of Central Florida, Orlando,
FL, USA}

\icmlcorrespondingauthor{HanQin Cai}{hqcai@ucf.edu}

\icmlkeywords{Learning to Optimize, Algorithm Unrolling}

\vskip 0.3in
]



\printAffiliationsAndNotice{\icmlEqualContribution} 


\begin{abstract}
Learning to Optimize (L2O), a technique that utilizes machine learning to learn an optimization algorithm automatically from data, has gained arising attention in recent years. A generic L2O approach parameterizes the iterative update rule and learns the update direction as a black-box network. While the generic approach is widely applicable, the learned model can overfit and may not generalize well to out-of-distribution test sets. In this paper, we derive the basic mathematical conditions that successful update rules commonly satisfy. Consequently, we propose a novel L2O model with a mathematics-inspired structure that is broadly applicable and generalized well to out-of-distribution problems. Numerical simulations validate our theoretical findings and demonstrate the superior empirical performance of the proposed L2O model.
\end{abstract}

\section{Introduction}
\label{sec:intro}
Solving mathematical problems with the help of artificial intelligence, particularly machine learning techniques, has gained increasing interest recently \citep{davies2021advancing,charton2021linear,polu2022formal,drori2021neural}. Optimization problems, a type of math problem that finds a point with minimal objective function value in a given space, can also be solved with machine learning models~\citep{gregor2010learning,andrychowicz2016learning,chen2021learning,bengio2021machine}. Such technique is coined as \emph{Learning to Optimize (L2O)}.

As an example, we consider an unconstrained optimization problem $\min_{\vx\in\mathbb{R}^n} F(\vx)$ where $F$ is differentiable. A classic algorithm to solve this problem is \emph{gradient descent}:
\[\vx_{k+1} = \vx_{k} - \alpha_{k} \nabla F(\vx_{k}),\quad k = 0,1,2,\cdots,\]
where the estimate of $\vx$ is updated in an iterative manner, $\alpha_{k}>0$ is a positive scalar named as step size, and
the update direction $\alpha_{k} \nabla F(\vx_{k})$ is aligned with the gradient of $F$ at $\vx_{k}$.
Instead of the vanilla gradient descent, \citep{andrychowicz2016learning} 
proposes to parameterize the update rule into a learnable model that suggests the update directions by taking the current estimate and the gradient of $F$ as inputs
\begin{equation}
    \label{eq:l2o}
   \vx_{k+1} = \vx_{k} - \vd_{k}(\vx_{k},\nabla F(\vx_{k});\phi),~k=0,1,\cdots,K-1,
\end{equation}
where $\phi$ is the learnable parameter that can be trained by minimizing a loss function: 
\begin{equation}
    \label{eq:loss}
\min_{\phi}\cL(\phi) := \mathbb{E}_{F\in\mathcal{F}} \Big[\sum_{k=1}^K w_{k}F(\vx_{k})\Big],
\end{equation}
where $\mathcal{F}$ is the problem set we concern and $\{w_{k}\}_{k=1}^{K}$ is a set of hand-tuned weighting coefficients. 
Such loss function aims at finding an update rule of $\vx_{k}$ such that the objective values $\{F(\vx_{k})\}$ are as small as possible for all $F\in\mathcal{F}$. 
This work and its following works~\citep{lv2017learning,wichrowska2017learned,wu2018understanding,metz2019understanding,chen2020training,shen2021learning,harrison2022closer} show that modeling $\vd_{k}$ with a deep neural network and learning a good update rule from data is doable. To train such models, they randomly pick some training samples from $\mathcal{F}$ and build estimates of the loss function defined in \eqref{eq:loss}. Such learned rules are able to generalize to unseen instances from $\mathcal{F}$, i.e., the problems similar to the training samples. This method is quite generic and we can use it as long as we can access the gradient or subgradient of $F$. For simplicity, we name the method in \eqref{eq:l2o} as \emph{generic L2O}.

Generic L2O is flexible and applicable to a broad class of problems. However, generalizing the learned update rules to out-of-distribution testing problems is quite challenging and a totally free $\vd_{k}$ usually leads to overfitting \citep{metz2020tasks,metz2022practical}. In this paper, we propose an approach to explicitly regularize the update rule $\vd_{k}$. Our motivation comes from some common properties that basic optimization algorithms should satisfy. For example, if an iterate $\vx_{k}$ reaches one of the minimizers of the objective $F(\vx)$, the next iterate $\vx_{k+1}$ should be fixed. Such condition is satisfied by many basic algorithms like gradient descent, but not necessarily satisfied if $\vd_{k}$ is free to choose. In this paper, we refer to an update rule as \emph{a good rule} if it fulfills these conditions.
Our main contributions are three-fold:
\begin{enumerate}
    \item We strictly describe some basic mathematical conditions that a good update rule should satisfy on convex optimization problems.
    \item Based on these conditions, we derive a math-inspired structure of the explicit update rule on $\vx_{k}$.
    \item We numerically validate that our proposed scheme has superior generalization performance. An update rule trained with randomly generated data can even perform surprisingly well on real datasets.
\end{enumerate}

\paragraph{Organization.} The rest of this paper is organized as follows. In Section~\ref{sec:cond}, we derive mathematical structures for L2O models. In Section~\ref{sec:algo}, we propose a novel L2O model and discuss its relationship with other L2O models. In Section~\ref{sec:exp}, we verify the empirical performance of the proposed model via numerical experiments. In Section~\ref{sec:conclusion}, we conclude the paper with some discussions on the future directions.

\section{Deriving Mathematical Structures for L2O Update Rule}
\label{sec:cond}
In this study, we consider optimization problems in the form of
$\min_{\vx\in\mathbb{R}^n} F(\vx) = f(\vx)  + r(\vx)$, where $f(\vx)$ is a smooth convex function with Lipschitz continuous gradient, and $r(\vx)$ is a convex function that may be non-smooth. More rigorously, we write that $f \in \cF_{L}(\mathbb{R}^n)$ and $r \in \cF(\mathbb{R}^n)$ where the function spaces $\cF(\mathbb{R}^n)$ and $\cF_{L}(\mathbb{R}^n)$ are defined below.
\begin{definition}[Spaces of Objective Functions]
\label{define:f-space}
 We define function spaces $\cF(\mathbb{R}^n)$ and $\cF_{L}(\mathbb{R}^n)$ as
\[
\begin{aligned}
 \cF(\mathbb{R}^n) &= \Big\{r: \mathbb{R}^n\to \mathbb{R} ~
 \Big|~  r \textnormal{ is proper, closed and convex} \Big\},\\
 \cF_{L}(\mathbb{R}^n) &= \Big\{f: \mathbb{R}^n\to \mathbb{R}~
 \Big|~  f\textnormal{ is convex, differentiable, and} \\
 &\quad~~ \|\nabla f(\vx)- \nabla f(\vy)\| \leq L\|\vx-\vy\|, \forall \vx,\vy \in \mathbb{R}^n \Big\}.
\end{aligned}\]
\end{definition}
The first derivative of $F$ plays an important role in the update rule \eqref{eq:l2o}. For differentiable function $f \in \cF_{L}(\mathbb{R}^n)$, we can access to its gradient $\nabla f(\vx)$. For non-differentiable function $r \in \cF(\mathbb{R}^n)$, we have to use the concepts of \emph{subgradient} and \emph{subdifferential} that are described below.
\begin{definition}[Subdifferential and Subgradient]
\label{define:subgra}
For $r \in \cF(\mathbb{R}^n)$, its subdifferential at $\vx$ is defined as
\[\partial r(\vx) = \big\{\vg \in \mathbb{R}^n~\big|~ r(\vy) - r(\vx) \geq \vg^\top (\vy - \vx),~\forall \vy \in \mathbb{R}^n\big\}.\]
Each element in the subdifferential, i.e., each $\vg \in \partial r(\vx)$, is a subgradient of function $r$ at point $\vx$.
\end{definition}

\paragraph{Settings of $\vd_{k}$.} We clarify some definitions about the update direction $\vd_{k}$. A general parameterized update rule can be written as
\begin{equation}
    \label{eq:rule-z}
    \vx_{k+1}=\vx_{k}-\vd_{k}(\vz_{k};\phi),
\end{equation}
where $\vz_{k}\in\cZ$ is the \emph{input vector} and $\cZ$ is the \emph{input space}. The input vector may involve dynamic information such as $\{\vx_{k}, F(\vx_{k}), \nabla F(\vx_{k})\}$. Take \citep{andrychowicz2016learning} as an example, as described in \eqref{eq:l2o}, the input vector is $\vz_{k} = [\vx_{k}^\top, \nabla F(\vx_{k})^\top ]^\top$ and the input space is $\cZ = \mathbb{R}^{2n}$. 
In our theoretical analysis,
we relax the structure of \eqref{eq:rule-z} and use a general update rule $\vd_{k}: \cZ\to\mathbb{R}^{n}$ instead of the parameterized rule $\vd_{k}(\vz_{k}; \phi)$ and write \eqref{eq:rule-z} as:
\begin{equation}
    \label{eq:l2o-general}
    \vx_{k+1} = \vx_{k} - \vd_{k}(\vz_{k}).
\end{equation}

To facilitate the theoretical analysis, we assume the update direction $\vd_{k}$ is differentiable with respect to the input vector, and its derivatives are bounded. Specifically, $\vd_{k}$ is taken from the space $\cD_{C}(\cZ)$ which is defined below.
\begin{definition}[Space of Update Rules]
\label{define:d-space}
Let $\mathrm{J} \vd(\vz)$ denote the Jacobian matrix of operator $\vd:\cZ\to\mathbb{R}^{n}$ and $\|\cdot\|_{\mathrm{F}}$ denote the Frobenius norm, we define the space:
\[
\begin{aligned}
 \cD_{C}(\cZ) = \Big\{\vd: \cZ \to \mathbb{R}^n ~\big|~  \vd \textnormal{ is differentiable,~~~} \\
\| \mathrm{J}\vd(\vz)\|_{\mathrm{F}} \leq C,~\forall \vz \in \cZ \Big\}.
\end{aligned}
\]
\end{definition}

In practice, training the deep network that is parameterized from $\vd_{k}$ will usually need the derivatives of $\vd_{k}$. Thus, the differentiability and bounded Jacobian of $\vd_k$ are important for this study. 
Note that $\vd_{k}\in\cD_{C}(\cZ)$ has been commonly used and satisfied in many existing parameterization approaches, e.g., Long Short-Term Memory (LSTM), which is one of the most popular models adopted in L2O \citep{andrychowicz2016learning,lv2017learning}.

\subsection{Smooth Case}
\label{sec:smooth}
In the smooth case, $\nabla F(\vx)$ equals to $\nabla f(\vx)$ as the non-smooth part $r(\vx)=0$. Thus, \eqref{eq:l2o-general} can be written as: 
\begin{equation}
    \label{eq:l2o-gd}
    \vx_{k+1} = \vx_{k} - \vd_{k}(\vx_{k},\nabla f(\vx_{k})).
\end{equation}
We refer \eqref{eq:l2o-gd} to be a good update rule if it satisfies the following two assumptions: 

\paragraph{Asymptotic Fixed Point Condition.} We assume that $\vx_{k+1}=\vx_{\ast}$ as long as $\vx_{k} = \vx_{\ast}$, where $\vx_\ast \in \argmin_{\vx} f(\vx)$. In other words, if $\vx_{k}$ is exactly a solution, the next iterate should be fixed. Substituting both $\vx_{k}$ and $\vx_{k+1}$ with $\vx_{\ast}$, we obtain: 
\[
\vx_{\ast} = \vx_{\ast} - \vd_{k}(\vx_{\ast},\nabla f(\vx_{\ast})),\quad \text{for all }k = 0,1,2,\cdots. 
\]
Convex analysis theory tells us $\nabla f(\vx_{\ast}) = \mathbf{0}$, and we obtain $\vd_{k}(\vx_{\ast},\mathbf{0}) = \mathbf{0}$. Instead of using this strong assumption, we relax it and assume $\vd_{k}(\vx_{\ast},\mathbf{0}) \to \mathbf{0}$ as $k\to\infty$. Formally, it is written as below and coined as \eqref{eq:fp1}:
    \begin{equation}
        \label{eq:fp1}\tag{FP1}
       \text{For any $\vx_\ast \in \argmin_{\vx\in\mathbb{R}^n}f(\vx)$, $\lim_{k\to\infty}\vd_{k}(\vx_{\ast},\mathbf{0}) = \mathbf{0}$.} 
    \end{equation}
    Such a condition can be viewed as an extension to the Fixed Point Encoding \citep{ryu2022large} in optimization, which is useful guidance for designing efficient convex optimization algorithms.

\paragraph{Global Convergence.} We assume that, the sequence $\{\vx_{k}\}_{k=0}^{\infty}$ converges to one of the minimizers of the objective function $f(\vx)$, as long as it yields the update rule \eqref{eq:l2o-gd}. Formally, it is written as \eqref{eq:gc1}:
    \begin{align}
         &\text{For any sequences $\{\vx_{k}\}_{k=0}^{\infty}$ generated by \eqref{eq:l2o-gd}, there exists} \nonumber \\
         &\text{$\vx_\ast \in \argmin_{\vx\in\mathbb{R}^n}f(\vx)$ such that $\lim_{k\to\infty}\vx_{k}= \vx_{\ast}$.}  \label{eq:gc1}\tag{GC1}
    \end{align}
Actually, assumptions \eqref{eq:fp1} and \eqref{eq:gc1} are fundamental in the field of optimization and can be satisfied by many basic update schemes. For example, as long as $f \in \cF_{L}(\mathbb{R}^n)$, gradient descent satisfies \eqref{eq:fp1} unconditionally and satisfies \eqref{eq:gc1} with a properly chosen step size. To outperform the vanilla update rules like gradient descent, a learned update rule $\vd_{k}$ should also satisfy \eqref{eq:fp1} and \eqref{eq:gc1}. 

The following theorem provides an analysis on $\vd_{k}$ under \eqref{eq:fp1} and \eqref{eq:gc1}. Note that proofs of all theorems in this section are deferred to the Appendix.
 \begin{theorem}
 \label{theo:gd}
 Given $f \in \cF_{L}(\mathbb{R}^n)$, we pick a sequence of operators $\{\vd_{k}\}_{k=0}^{\infty}$ with $\vd_{k}\in \cD_{C}(\mathbb{R}^{2n})$ and generate  $\{\vx_{k}\}_{k=0}^{\infty}$ by \eqref{eq:l2o-gd}. 
If both \eqref{eq:fp1} and \eqref{eq:gc1} hold, then for all $k=0,1,2,\cdots$, there exist $\vP_{k} \in \mathbb{R}^{n \times n}$ and $\vb_{k} \in \mathbb{R}^{n}$ satisfying 
\begin{equation*}
    \vd_{k}(\vx_{k},\nabla f(\vx_{k})) = \vP_{k}\nabla f(\vx_{k}) + \vb_{k},
\end{equation*}
with $\vP_{k}$ is bounded and $\vb_{k}\to \mathbf{0}$ as $k\to\infty$.
 \end{theorem}
Theorem~\ref{theo:gd} illustrates that an update rule $\vd_{k}$ is not completely free under assumptions \eqref{eq:fp1} and \eqref{eq:gc1}. It suggests the following structured update rule instead of the free update rule \eqref{eq:l2o-gd}:
\begin{equation} \label{eq:l2o-pb}
    \vx_{k+1} = \vx_{k} - \vP_{k}\nabla f(\vx_{k}) - \vb_{k},
\end{equation}
where $\vP_{k}$ is named as a \emph{preconditioner} and $\vb_{k} $ is named as a \emph{bias}. The scheme \eqref{eq:l2o-pb} covers several classical algorithms. 
 For example, with $\vP_{k} = \alpha \mathbf{I}$ and $\vb_{k} = \beta(\vx_{k} - \alpha \nabla f(\vx_{k}) - \vx_{k-1} - \alpha\nabla f(\vx_{k-1}))$, it reduces to Nesterov accelerated gradient descent~\citep{nesterov1983method}; with $\vb_{k} = \mathbf{0}$ and properly chosen $\vP_{k}$, it covers Newton's method and quasi-Newton method like L-BFGS~\citep{liu1989limited}.
 
 Furthermore, Theorem~\ref{theo:gd} implies that, as long as \eqref{eq:fp1} and \eqref{eq:gc1} are both satisfied, finding the optimal update direction $\vd_{k}$ equals to finding the optimal preconditioner and bias.  To train an update rule for smooth convex objective functions $f \in \cF_{L}$, one may parameterize $\vP_{k}$ and $\vb_{k}$ instead of parameterizing the entire $\vd_{k}$, that is,
\[ \vx_{k+1} = \vx_{k} - \vP_{k}\big(\vz_{k}; \phi\big) \nabla f(\vx_{k}) - \vb_{k}\big(\vz_{k}; \psi \big),\]
where the input vector $\vz_{k} = [\vx_{k}^\top,\nabla f(\vx_{k})^\top]^\top$. Detailed parameterization and training methods are later described in Section~\ref{sec:algo}.

Note that even if the update rule satisfies \eqref{eq:l2o-pb} instead of \eqref{eq:l2o-gd}, one cannot guarantee \eqref{eq:fp1} and \eqref{eq:gc1} hold. Under our assumptions, if one further makes assumptions on the trajectory of $\{x_{k}\}_{k=0}^{K}$, then they would obtain a convergence guarantee. This idea falls into a subfield of optimization called Performance Estimation Problem (PEP) \citep{taylor2017exact,ryu2020operator}. However, algorithms obtained by PEP are much slower than L2O on specific types of problems that a single user is concerned with. This is because PEP imposes many restrictions on the iteration path, while L2O can find shortcuts for particular problem types. 

Although we also impose restrictions (Asymptotic Fixed Point Condition and Global Convergence) on L2O, these restrictions target the asymptotic performance and the final fixed point, rather than the path. As a result, the math-structured L2O proposed in this paper avoids the limitations of convergence guarantees while allowing shortcuts. By analyzing the fixed point, we discover a more effective learnable optimizer structure. 
 
 \subsection{Non-Smooth Case}
 \label{sec:non-smooth}
 In the non-smooth case (i.e., $f(\vx)=0$), we use subgradient instead of gradient. A simple extension to \eqref{eq:l2o-gd} is to pick a subgradient $\vg_{k}$ from subdifferential $\partial r(\vx_{k})$ in each iteration and use it in the update rule:
 \begin{equation}
    \label{eq:l2o-subgd}
    \vx_{k+1} = \vx_{k} - \vd_{k}(\vx_{k},\vg_{k}),\quad \vg_{k} \in \partial r(\vx_{k}).
\end{equation}
Such an update rule is an extension to subgradient descent method: $\vx_{k+1}=\vx_{k}-\alpha_{k}\vg_{k}$. Compared with gradient descent in the smooth case, the convergence of subgradient descent method is usually  unstable, and it may not converge to the solution if a constant step size is used. To guarantee convergence, one has to use certain diminishing step sizes, which may lead to slow convergence. \citep{bertsekas2015convex}. 
For non-smooth problems, Proximal Point Algorithm (PPA)~\citep{rockafellar1976monotone} usually converges faster and more stably than the subgradient descent method. While subgradient descent method uses the \emph{explicit update}, PPA takes \emph{implicit update rule}: $\vx_{k+1}=\vx_{k}-\alpha_{k}\vg_{k+1}$, where $\vg_{k+1} \in \partial r(\vx_{k+1})$. Inspired by PPA, we propose to use the following implicit rule instead: 
\begin{equation}
    \label{eq:l2o-ppa}
    \vx_{k+1} = \vx_{k} - \vd_{k}(\vx_{k+1},\vg_{k+1}),\quad \vg_{k+1} \in \partial r(\vx_{k+1}).
\end{equation}
Generally speaking, it is hard to calculate $\vx_{k+1}$ from the implicit formula \eqref{eq:l2o-ppa} given $\vx_{k}$. However, the following discussion provides a mathematical structure of $\vd_{k}$ in \eqref{eq:l2o-ppa}, and, under some mild assumptions, \eqref{eq:l2o-ppa} can be written in a much more practical way.

With the same argument in the smooth case, we obtain the math description of the asymptotic fixed point condition: For any $\vx_\ast \in \argmin_{\vx} r(\vx)$, there exists a $\vg_{\ast} \in \partial r(\vx_{\ast})$ such that $\vd_{k}(\vx_{\ast},\vg_{\ast}) \to \mathbf{0}$ as $k \to \infty$. With convex analysis theory, it holds that $\mathbf{0} \in \partial r(\vx_{\ast})$ if and only if $\vx_\ast \in \argmin_{\vx} r(\vx)$. Thus, it is natural to take $\vg_{\ast} = \mathbf{0}$. Formally, it is written as
    \begin{equation}
        \label{eq:fp2}\tag{FP2}
       \text{For any $\vx_\ast \in \argmin_{\vx\in\mathbb{R}^n}r(\vx)$,  $\lim_{k\to\infty}\vd_{k}(\vx_{\ast},\mathbf{0}) = \mathbf{0}$.} 
    \end{equation}
    Similar to \eqref{eq:gc1}, in the non-smooth case, we require the sequence $\{\vx_{k}\}$ converges to one of the minimizers of the function $r(\vx)$. Formally, it is written as
    \begin{align}
         &\text{For any sequences $\{\vx_{k}\}_{k=0}^{\infty}$ generated by \eqref{eq:l2o-ppa}, there exists } \nonumber \\
         &\text{$\vx_\ast \in \argmin_{\vx\in\mathbb{R}^n}r(\vx)$ such that $\lim_{k\to\infty}\vx_{k}= \vx_{\ast}$.} \label{eq:gc2}\tag{GC2}
    \end{align}

\begin{theorem}
\label{theo:ppa}
Given $r \in \cF(\mathbb{R}^n)$, we pick a sequence of operators $\{\vd_{k}\}_{k=0}^{\infty}$ with $\vd_{k}\in \cD_{C}(\mathbb{R}^{2n})$ and generate $\{\vx_{k}\}_{k=0}^{\infty}$ by \eqref{eq:l2o-ppa}. 
If both \eqref{eq:fp2} and \eqref{eq:gc2} hold, 
then for all $k=0,1,2,\cdots$, there exist $\vP_{k} \in \mathbb{R}^{n \times n}$ and $\vb_{k} \in \mathbb{R}^{n}$ satisfying 
\[\vx_{k+1} = \vx_{k} - \vP_{k}\vg_{k+1} - \vb_{k}, \quad \vg_{k+1} \in \partial r(\vx_{k+1}),\]
with $\vP_{k}$ is bounded and $\vb_{k}\to \mathbf{0}$ as $k\to\infty$. If we further assume $\vP_{k}$ is symmetric positive definite, then $\vx_{k+1}$ can be uniquely determined through
\begin{equation}
    \label{eq:l2o-ppa-quadratic}
    \vx_{k+1} = \argmin_{\vx \in \mathbb{R}^n} r(\vx) + \frac{1}{2}\|\vx - \vx_{k} + \vb_{k}\|^2_{\vP_{k}^{-1}},
\end{equation}
where the norm $\|\cdot\|_{\vP_{k}^{-1}}$ is defined as $\|\vx\|_{\vP_{k}^{-1}} := \sqrt{\vx^\top \vP_{k}^{-1} \vx}$.
\end{theorem}

Define the preconditioned \emph{proximal operator} of $r(\vx)$ as 
\begin{equation}
    \label{eq:define-prox}
    \prox_{r, \vP}(\Bar{\vx}):= \argmin_{\vx} r(\vx) + \frac{1}{2}\|\vx - \Bar{\vx} \|^2_{\vP^{-1}},
\end{equation}
where $\vP$ is a symmetric positive definite preconditioner.  Then \eqref{eq:l2o-ppa-quadratic} can be written as $\vx_{k+1}=\prox_{r,\vP_{k}}(\vx_{k}-\vb_{k})$. If we set $\vP_{k} = \vI$ and $\vb_{k}=\mathbf{0}$, \eqref{eq:l2o-ppa-quadratic} reduces to the standard PPA. Instead of learning a free update rule as shown in \eqref{eq:l2o-ppa}, Theorem~\ref{theo:ppa} suggests learning the preconditioner $\vP_{k}$ and the bias $\vb_{k}$ in a structured rule, as illustrated in equation \eqref{eq:l2o-ppa-quadratic}.

\subsection{Composite Case}
\label{sec:composite}

As special cases, the smooth case and non-smooth case provide important preliminaries to the composite case: $F(\vx)=f(\vx)+r(\vx)$.
Inspired by Theorems~\ref{theo:gd} and \ref{theo:ppa}, 
we use explicit formula for $f$ and implicit formula for $r$ in the composite case:
\begin{equation}
    \label{eq:l2o-fbs}
    \vx_{k+1} = \vx_{k} - \vd_{k}(\vx_{k},\nabla f(\vx_{k}),\vx_{k+1},\vg_{k+1}),
\end{equation}
where $\vg_{k+1} \in \partial r(\vx_{k+1})$ and the input vector $\vz_{k} = [\vx_{k}^\top,\nabla f(\vx_{k})^\top,\vx_{k+1}^\top,\vg_{k+1}^\top]^\top$ and input space is $\cZ = \mathbb{R}^{4n}$.

To derive the asymptotic fixed point condition in this case, we use the same arguments in Sections~\ref{sec:smooth} and \ref{sec:non-smooth}, and we obtain the following statement for all $\vx_{\ast}\in \argmin_{\vx} F(\vx)$:
\[
    \lim_{k\to \infty} \vd_{k}(\vx_{\ast},\nabla f(\vx_{\ast}),\vx_{\ast},\vg_{\ast}) = \mathbf{0}, ~ \text{for some }\vg_{\ast} \in \partial r(\vx_{\ast}).
\]
The convexity of $f$ and $r$ implies that $\mathbf{0} \in \nabla f(\vx_{\ast}) + \partial r(\vx_{\ast})$ if and only if $\vx_{\ast}\in \argmin_{\vx} F(\vx)$. Thus, it holds that $-\nabla f(\vx_{\ast}) \in \partial r(\vx_{\ast})$. With $\vg_{\ast} = -\nabla f(\vx_{\ast})$, one could obtain the formal statement of the fixed point condition. For any $\vx_\ast \in \argmin_{\vx\in\mathbb{R}^n}F(\vx)$, it holds that
\begin{align}
 \lim_{k\to \infty} \vd_{k}(\vx_{\ast},\nabla f(\vx_{\ast}),\vx_{\ast},-\nabla f(\vx_{\ast})) = \mathbf{0}. \label{eq:fp3}\tag{FP3}
\end{align}
The global convergence is stated as:
\begin{align}
         &\text{For any sequences $\{\vx_{k}\}_{k=0}^{\infty}$ generated by \eqref{eq:l2o-fbs}, there exists } \nonumber \\
         &\text{$\vx_\ast \in \argmin_{\vx\in\mathbb{R}^n}F(\vx)$ such that $\lim_{k\to\infty}\vx_{k}= \vx_{\ast}$.} \label{eq:gc3}\tag{GC3}
\end{align}

\begin{theorem}
\label{theo:fbs}
Given $f \in \cF_{L}(\mathbb{R}^n)$ and $r \in \cF(\mathbb{R}^n)$, we pick a sequence of operators $\{\vd_{k}\}_{k=0}^{\infty}$ with $\vd_{k}\in \cD_{C}(\mathbb{R}^{4n})$ and generate $\{\vx_{k}\}_{k=0}^{\infty}$ by \eqref{eq:l2o-fbs}. 
If both \eqref{eq:fp3} and \eqref{eq:gc3} hold,
 then for all $k=0,1,2,\cdots$, there exist $\vP_{k} \in \mathbb{R}^{n \times n}$ and $\vb_{k} \in \mathbb{R}^{n}$ satisfying 
\[\vx_{k+1} = \vx_{k} - \vP_{k}(\nabla f(\vx_{k}) - \vg_{k+1}) - \vb_{k},~\vg_{k+1} \in \partial r(\vx_{k+1}),\]
with $\vP_{k}$ is bounded and $\vb_{k}\to \mathbf{0}$ as $k\to\infty$. If we further assume $\vP_{k}$ is symmetric positive definite, then $\vx_{k+1}$ can be uniquely determined given $\vx_{k}$ through
\begin{equation}
    \label{eq:l2o-fbs-quadratic}
    \vx_{k+1} = \prox_{r, \vP_{k}}(\vx_{k} - \vP_{k}\nabla f(\vx_{k}) - \vb_{k}).
\end{equation}
\end{theorem}
 
With $\vb_{k}=\mathbf{0}$ and $\vP_{k} = \alpha \mathbf{I}$, \eqref{eq:l2o-fbs-quadratic} reduces to a standard Proximal Gradient Descent (PGD). 
Therefore, scheme \eqref{eq:l2o-fbs-quadratic} is actually an extension of PGD with a preconditioner $\vP_{k}$ and a bias $\vb_{k}$. Theorem~\ref{theo:fbs} implies that it's enough to learn an extended PGD instead of a free scheme \eqref{eq:l2o-fbs}.

\subsection{Longer Horizon}
 
Those update schemes \eqref{eq:l2o-gd}, \eqref{eq:l2o-ppa} and \eqref{eq:l2o-fbs} introduced in previous sections  explicitly depend on only the current status $\vx_{k}$. Now we introduce an auxiliary variable $\vy_{k}$ that encodes historical information through operator $\vm$:
 \begin{equation}
     \label{eq:l2o-nesterov-y}
     \vy_{k} = \vm(\vx_{k},\vx_{k-1},\cdots,\vx_{k-T}).
 \end{equation}
To facilitate parameterization and training, we assume $\vm$ is differentiable: $\vm \in \cD_{C}(\mathbb{R}^{ (T+1)\times n})$ (see Definition~\ref{define:d-space}). With $\vy_{k}$, we could encode more information into the update rule and extend \eqref{eq:l2o-fbs} to the following:
 \begin{align}
& \vx_{k+1}  = \vx_{k} - \vd_{k}(\vx_{k},\nabla f(\vx_{k}),\vx_{k+1},\vg_{k+1},\vy_{k},\nabla f(\vy_{k})), \nonumber \\
 & \text{where }~~\vg_{k+1} \in \partial r(\vx_{k+1}). \label{eq:l2o-nesterov-x}
 \end{align}
 Now let's derive the asymptotic fixed point condition and the global convergence that \eqref{eq:l2o-nesterov-y} and \eqref{eq:l2o-nesterov-x} should follow. Since the global convergence requires $\vx_{k}\to\vx_{\ast}$, the continuity of operator $\vm$ implies the convergence of sequence $\{\vy_{k}\}$. If the limit point of $\{\vy_{k}\}$ is denoted by $\vy_{\ast}$, we can assume $\vy_{\ast} = \vx_{\ast}$ without loss of generality because, for any operator $\vm$,  we can always construct another operator by shifting the output: $\hat{\vm} = \vm -\vy_{\ast} + \vx_{\ast}$ such that the sequence generated by $\hat{\vm}$ converges to $\vx_{\ast}$. Roughly speaking, we assume that the sequence $\{\vx_{k},\vy_{k}\}$ generated by \eqref{eq:l2o-nesterov-y} and \eqref{eq:l2o-nesterov-x} satisfies $\vx_{k}\to\vx_{\ast}$ and $\vy_{k}\to\vx_{\ast}$. 
 By extending (\ref{eq:fp3}) and (\ref{eq:gc3}), 
 we obtain the formal statement of our assumptions that \eqref{eq:l2o-nesterov-y} and \eqref{eq:l2o-nesterov-x} should follow.

For any $\vx_\ast \in \argmin_{\vx\in\mathbb{R}^n}F(\vx)$, it holds that 
 \begin{align}
   \lim_{k\to \infty} & \vd_{k}(\vx_{\ast},\nabla f(\vx_{\ast}),\vx_{\ast},-\nabla f(\vx_{\ast}),\vx_{\ast},\nabla f(\vx_{\ast})) = \mathbf{0}, \nonumber\\
    & \vm(\vx_{\ast},\vx_{\ast},\cdots,\vx_{\ast}) = \vx_{\ast}. \label{eq:fp4}\tag{FP4}
\end{align}
For any sequences $\{\vx_{k}, \vy_{k}\}_{k=0}^{\infty}$ generated by \eqref{eq:l2o-nesterov-y} and \eqref{eq:l2o-nesterov-x}, there exists one $\vx_\ast \in \argmin_{\vx\in\mathbb{R}^n}F(\vx)$ such that 
\begin{equation}
    \label{eq:gc4}\tag{GC4}
\lim_{k\to\infty}\vx_{k} = \lim_{k\to\infty}\vy_{k} = \vx_{\ast}.
\end{equation}

\begin{theorem}
\label{theo:nesterov} Suppose $T=1$. 
Given $f \in \cF_{L}(\mathbb{R}^n)$ and $r \in \cF(\mathbb{R}^n)$, we pick an operator $\vm \in \cD_{C}(\mathbb{R}^{2n})$ and a sequence of operators $\{\vd_{k}\}_{k=0}^{\infty}$ with $\vd_{k}\in \cD_{C}(\mathbb{R}^{6n})$. 
If both \eqref{eq:fp4} and \eqref{eq:gc4} hold, for any bounded matrix sequence $\{\vB_{k}\}_{k=0}^{\infty}$, 
 there exist $\vP_{1,k},\vP_{2,k},\vA_{k} \in \mathbb{R}^{n \times n}$ and $\vb_{1,k},\vb_{2,k} \in \mathbb{R}^{n}$ satisfying 
\begin{align}
      \vx_{k+1}  & = \vx_{k} - (\vP_{1,k}-\vP_{2,k})\nabla f(\vx_{k}) - \vP_{2,k}\nabla f(\vy_{k}) - \vb_{1,k} \nonumber \\ 
      &- \vP_{1,k} \vg_{k+1} - \vB_{k}(\vy_{k} - \vx_{k}),~\vg_{k+1} \in \partial r(\vx_{k+1}), \label{eq:l2o-nesterov-d}\\ 
  \vy_{k+1} & = (\vI - \vA_{k})\vx_{k+1} + \vA_{k}\vx_{k} + \vb_{2,k} \label{eq:l2o-nesterov-m}
\end{align}
for all $k=0,1,2,\cdots$, with $\{\vP_{1,k},\vP_{2,k},\vA_{k}\}$ are bounded and $\vb_{1,k}\to \mathbf{0}, \vb_{2,k}\to \mathbf{0}$ as $k\to\infty$. If we further assume $\vP_{1,k}$ is uniformly symmetric positive definite\footnote{A sequence of uniformly symmetric positive definite matrices means that the smallest eigenvalues of all symmetric positive definite matrices are uniformly bounded away from zero.}, then we can substitute $\vP_{2,k}\vP_{1,k}^{-1}$ with $\vB_{k}$ and obtain
\begin{equation}
    \label{eq:l2o-nesterov-quadratic}
    \begin{aligned}
    \hat{\vx}_{k} &= \vx_{k} - \vP_{1,k}\nabla f(\vx_{k}),\\
    \hat{\vy}_{k} &= \vy_{k} - \vP_{1,k}\nabla f(\vy_{k}),\\
     \vx_{k+1} &= \prox_{r, \vP_{1,k}} \Big( (\vI - \vB_{k})\hat{\vx}_{k} + \vB_{k} \hat{\vy}_{k} - \vb_{1,k} \Big),\\
     \vy_{k+1} &= \vx_{k+1} + \vA_{k} (\vx_{k+1} - \vx_{k}) + \vb_{2,k}.
    \end{aligned}
\end{equation}
\end{theorem}
 
In the update scheme \eqref{eq:l2o-nesterov-quadratic}, $\vb_{1,k}$ and $\vb_{2,k}$ are biases that play the same role with $\vb_{k}$ in \eqref{eq:l2o-fbs-quadratic}; $\vA_{k}$ can be viewed as an extension of Nesterov momentum and we name it as an \emph{accelerator}; $\vP_{1,k}$ is the preconditioner that plays a similar role as $\vP_{k}$ in \eqref{eq:l2o-fbs-quadratic}; $\vB_{k}$ is a balancing term between $\hat{\vx}_{k}$ and $\hat{\vy}_{k}$. If $\vB_{k}=\mathbf{0}$, then $\vx_{k+1}$ merely depends on $\vx_{k}$ and \eqref{eq:l2o-nesterov-quadratic} reduces to \eqref{eq:l2o-fbs-quadratic}; and if $\vB_{k}=\vI$, then $\vx_{k+1}$ merely depends on $\vy_{k}$ explicitly.

\section{An Efficient Math-Inspired L2O Model}
\label{sec:algo}

As long as the basic assumptions \eqref{eq:fp4} and \eqref{eq:gc4} hold, one could derive a math-structured update rule \eqref{eq:l2o-nesterov-quadratic} from generic update rule \eqref{eq:l2o-nesterov-y} and \eqref{eq:l2o-nesterov-x}.
Moreover, we suggest using diagonal matrices for $\vP_{1,k},\vB_{k},\vA_{k}$ in practice: 
\[\vP_{1,k} = \mathrm{diag}(\vp_{k}),~~\vB_{k} = \mathrm{diag}(\vb_{k}),~~\vA_{k} = \mathrm{diag}(\va_{k}),\]where $\vp_{k},\vb_{k},\va_{k} \in \mathbb{R}^{n}$ are vectors. Let $\mathbf{1}$ be the vector whose all elements are ones and $\odot$ be the element-wise multiplication. The suggested update rule then becomes:
\begin{equation}
    \label{eq:final-scheme}
     \begin{aligned}
     \hat{\vx}_{k} &= \vx_{k} - \vp_{k}\odot\nabla f(\vx_{k}),\\
     \hat{\vy}_{k} &= \vy_{k} - \vp_{k}\odot\nabla f(\vy_{k}),\\
      \vx_{k+1} &= \prox_{r, \vp_{k}} \Big( (\mathbf{1} - \vb_{k}) \odot \hat{\vx}_{k} + \vb_{k} \odot \hat{\vy}_{k} - \vb_{1,k} \Big),\\
      \vy_{k+1} &= \vx_{k+1} + \va_{k} \odot (\vx_{k+1} - \vx_{k}) + \vb_{2,k}.
    \end{aligned}
\end{equation}
%
We choose diagonal $\vP_{1,k},\vB_{k},\vA_{k}$ over full matrices for efficiency.
On one hand, the diagonal formulation reduces the degree of freedom of the update rule. Therefore, when $\vP_{1,k},\vB_{k},\vA_{k}$ are parameterized as the output of, or a part of, a learnable model and trained with data, the difficulty of training is decreased and thus the efficiency improved.
On the other hand, for a broad class of $r(\vx)$, the proximal operator $\prox_{r,\vp_{k}}$ has efficient explicit formula with the diagonal preconditioner\footnote{
Examples can be found in Appendix~\ref{sec:example-prox}.
} 
$\vp_{k}$. Although the non-diagonal formulation may lead to a (theoretically) better convergence rate, it could increase the computational difficulty of the proximal operator $\prox_{r, \vP_{k}}$. An interesting future topic is how to calculate \eqref{eq:final-scheme} with non-diagonal $\vP_{k}$ and how to train deep models to generate such $\vP_{k}$.

\paragraph{LSTM Parameterization.} Similar to \citep{andrychowicz2016learning,lv2017learning}, we model $\vp_{k}$, $\va_{k}$, $\vb_{k}$, $\vb_{1,k}$, $\vb_{2,k}$ as the output of a coordinate-wise LSTM, which is parameterized by learnable parameters $\phi_\text{LSTM}$ and takes the current estimate $\vx_k$ and the gradient $\nabla f(\vx_{k})$ as the input:
\begin{equation}
    \label{eq:final-scheme-lstm}
    \begin{aligned}
     \vo_{k},\vh_{k} = \mathrm{LSTM}\big(\vx_{k},\nabla f(\vx_{k}),\vh_{k-1}; \phi_\text{LSTM} \big),&\\
     \vp_{k}, \va_{k}, \vb_{k}, \vb_{1,k}, \vb_{2,k}= \mathrm{MLP}(\vo_{k}; \phi_{\text{MLP}}).&\\
    \end{aligned}
\end{equation}
Here, $\vh_k$ is the internal state maintained by the LSTM with $\vh_0$ randomly sampled from Gaussian distribution.
It is common in classic optimization algorithms to take positive $\vp_k$ and $\va_k$. Hence we post-process $\vp_k$ and $\va_k$ with an additional activation function such as sigmoid and softplus.
A ``coordinate-wise'' LSTM means that the same network is shared across all coordinates of $\vx_k$, so that this single model can be applied to optimization problems of any scale. \eqref{eq:final-scheme} and \eqref{eq:final-scheme-lstm} together define an optimization scheme. We call it an \emph{L2O optimizer}.

\paragraph{Training.} We train the proposed L2O optimizer, that is to find the optimal $\phi_\text{LSTM}$ and $\phi_\text{MLP}$ in \eqref{eq:final-scheme-lstm}, on a dataset $\mathcal{F}$ of optimization problems. Each sample in $\mathcal{F}$ is an instance of the optimization problem, which we call an \emph{optimizee}, and is characterized by its objective function $F$. During training, we apply optimizer to each optimizee $F$ for $K$ iterations to generate a sequence of iterates $(\vy_1, \dots, \vy_K)$, and optimize $\phi_\text{LSTM}$ and $\phi_\text{MLP}$ by minimizing the following loss function:
\begin{equation*}
    \label{eq:loss-train}
\min_{\phi_\text{LSTM},\phi_\text{MLP}}\cL(\phi_\text{LSTM},\phi_\text{MLP}) := \frac{1}{|\mathcal{F}|}\sum_{F\in\mathcal{F}} \Big[\frac{1}{K}\sum_{k=1}^K F(\vy_{k})\Big].
\end{equation*}


\paragraph{Compared with Algorithm Unrolling.} Algorithm Unrolling \citep{monga2019algorithm} is another line of works parallel to generic L2O. It was first proposed to fast approximate the solution of sparse coding \citep{gregor2010learning} which is named Learned ISTA (LISTA). Since then, many efforts have been made to further improve or better understand LISTA \citep{xin2016maximal,metzler2017learned,moreau2017understanding,chen2018theoretical,liu2019alista,ito2019trainable,chen2021hyperparameter}, as well as applying this idea to different optimization problems \citep{yang2016deep,zhang2018ista,adler2018learned,mardani2018neural,gupta2018cnn,solomon2019deep,xie2019differentiable,cai2021learned}.

The main difference between our method and algorithm-unrolling methods lies at how parameterization is done. Different from the LSTM parameterization \eqref{eq:final-scheme-lstm}, algorithm-unrolling methods turn $\vp_{k}$, $\va_{k}$, $\vb_{k}$, $\vb_{1,k}$, $\vb_{2,k}$ themselves as learnable parameters and directly optimize them from data.

However, such direct parameterization causes limitations on the flexibility of the model in many ways. It loses the ability to capture dynamics between iterations and tends to memorize more about the datasets. Moreover, direct parameterization means that we need to match the dimensions of the learnable parameters with the problem scale, which implies that the trained model can not be applied to optimization problems of a different scale at all during inference time. Although this can be worked around by reducing the parameters to scalars, it will significantly decrease the capacity of the model.

In fact, despite the difference in parameterization, our proposed scheme \eqref{eq:final-scheme} covers many existing algorithm-unrolling methods in the literature.
For example, 
if we use the standard LASSO objective as the objective function $F(\vx)$ and set $\va_{k}=\vb_{k}=\vb_{2,k}=\mathbf{0}$, we will recover Step-LISTA~\citep{ablin2019learning} with properly chosen $\vp_{k},\vb_{1,k}$.
More details and proofs are provided in the Appendix. 

\paragraph{Compared with Generic L2O.} Our proposed method uses a similar coordinate-wise LSTM parameterization as generic L2O methods. Therefore, both of these two share the flexibility of being applied to optimization problems of any scale. However, we constrain the update rule to have a specific form, i.e., the formulation in \eqref{eq:final-scheme}. The reduced degree of freedom enables the convergence analysis in Section~\ref{sec:cond} from the theoretical perspective and empirically works as a regularization so that the trained L2O optimizer is more stable and can generalize better, which is validated by our numerical observations in the next section.

We summarize in Table~\ref{tab:comparison-methods} the comparison between classic algorithms, algorithm-unrolling methods, generic L2O methods, and our math-inspired method in terms of theoretical convergence analysis, convergence speed, and flexibility.

\begin{table}[t]
    \vspace{-0.5em}
    \caption{Comparison between different types of methods of their theoretical convergence analysis (\textbf{Theory}), convergence speed (\textbf{Fast}), and \textbf{Flexibility}.} \label{tab:comparison-methods}
    \vspace{0.5em}
    \centering
    \begin{tabular}{c|ccc}
        \toprule
        Methods              & Theory        & Fast          & Flexibility  \\
        \midrule
        Classic Algorithms   & $\checkmark$  &    --         & $\checkmark$ \\
        Algorithm Unrolling  & $\checkmark$  & $\checkmark$  &        --    \\
        Generic L2O          &      --       & $\checkmark$  & $\checkmark$ \\
        Math-Inspired (Ours) & $\checkmark$  & $\checkmark$  & $\checkmark$ \\
        \bottomrule
    \end{tabular}   
\end{table}

\paragraph{Relationship with Operator Learning.} In this paper, we consider the proximal operator as an accessible basic routine and focus on learning the overall update rule that results in fast convergence. However, if the proximal operator is difficult to compute, one might explore another aspect of L2O: learning fast approximations of proximal operators \citep{zhang2017learning,meinhardt2017learning,li2022learning}. For instance, if the matrix $\vP_{1,k}$ in \eqref{eq:l2o-nesterov-quadratic} is not diagonal, calculating the proximal operator would be challenging. Additionally, our assumptions \eqref{eq:fp4} and \eqref{eq:gc4} allow the optimization problem $F(\vx)$ to have multiple optima. As a result, investigating diverse optima following the idea of \citep{li2022learning} using our proposed scheme could be an intriguing topic for future research.

\paragraph{Relationship with Meta-Learning.} L2O and Meta-Learning are closely related topics as they both deal with learning from experience in previous tasks to improve performance on new tasks. L2O treats tasks as optimization problems and aims to discover superior optimization algorithms, while Meta-Learning is a more comprehensive concept that focuses on training a model on a set of related tasks or problems to swiftly adapt to new, unseen tasks using knowledge acquired from prior tasks. For example, in our paper's equation \eqref{eq:l2o}, L2O seeks to learn $\vd_{k}$ while keeping the initialization $\vx_{0}$ fixed or randomized. Meanwhile, a typical Meta-Learning method like \citep{khodak2019provable} learns a suitable initialization from a series of observed tasks, enabling quick adaptation to unseen tasks. In addition to the initialization, \citep{khodak2019adaptive} also learns a meta-learning rate shared among different tasks. Furthermore, one can learn a regularization term based on the distance to a bias vector \citep{denevi2019learning} or even a conditional regularization term \citep{denevi2020advantage} from a set of tasks.

\section{Experimental Results}
\label{sec:exp}

We strictly follow the setting in \citep{lv2017learning} for experiment setup. More specifically, in all our experiments on the LSTM-based L2O models (including our method and other baseline competitors), we use two-layer LSTM cells with 20 hidden units with sigmoid activation functions. During training, each minibatch contains 64 instances of optimization problems, to which the L2O optimizers will be applied for 100 iterations. The 100 iterations are evenly segmented into 5 periods of 20 iterations. Within each of these, the L2O optimizers are trained with truncated Backpropagation Through Time (BPTT) with an Adam optimizer. All models are trained with 500 minibatches (32,000 optimization problems in total) generated synthetically, but are evaluated on both synthesized testing sets and real-world testing sets. We elaborate more on the data generation in the Appendix. The code is available online at \url{https://github.com/xhchrn/MS4L2O}.

\subsection{Ablation Study} 
\label{sec:abla}

We conduct an ablation study on LASSO to figure out the roles of $\vp_{k},\va_{k},\vb_{k},\vb_{1,k},\vb_{2,k}$ in our proposed scheme \eqref{eq:final-scheme}. Both the training and testing samples are independently sampled from the same random distribution. The form of LASSO is given below
\begin{equation}
\label{eq:lasso}
    \min_{\vx\in\mathbb{R}^n} F(\vx) = \frac{1}{2}\|\vA\vx-\vb\|^2 + \lambda \|\vx\|_1,
\end{equation}
where each tuple of $(\vA,\vb,\lambda)$ determines an objective function and thus a LASSO problem instance. The size of each instance is $\vA \in \mathbb{R}^{250\times500}$ and $\vb\in\mathbb{R}^{500}$ and other details are provided in the Appendix. We do not fix $\vA$ and, instead, let each LASSO instance take an independently generated $\vA$. This setting is fundamentally more challenging than those in most of algorithm-unrolling works \citep{gregor2010learning,liu2019alista,ablin2019learning,behrens2021neurally}.

On the benchmark, we compare the following settings: \textbf{PBA12}: $\vp_{k}$,$\va_{k}$,$\vb_{k}$,$\vb_{1,k}$,$\vb_{2,k}$ are all learnable.
    \textbf{PBA1}: $\vp_{k}$,$\va_{k}$,$\vb_{k}$,$\vb_{1,k}$ are learnable; $\vb_{2,k}$ is fixed as $\mathbf{0}$.
    \textbf{PBA2}: $\vp_{k}$,$\va_{k}$,$\vb_{k}$,$\vb_{2,k}$ are learnable; $\vb_{1,k}$ is fixed as $\mathbf{0}$.
    \textbf{PBA}: $\vp_{k}$,$\va_{k}$,$\vb_{k}$ are learnable; $\vb_{2,k}$ and $\vb_{1,k}$ are both fixed as $\mathbf{0}$.
    \textbf{PA}: $\vp_{k}$,$\va_{k}$ are learnable; $\vb_{2,k}$ and $\vb_{1,k}$ are both fixed as $\mathbf{0}$; $\vb_{k}$ is fixed as $\mathbf{1}$.
    \textbf{P}: only $\vp_{k}$ is learnable; $\va_{k}$, $\vb_{2,k}$, $\vb_{1,k}$ are fixed as $\mathbf{0}$; $\vb_{k}$ is fixed as $\mathbf{1}$.
    \textbf{A}: only $\va_{k}$ is learnable; $\vb_{2,k}$ and $\vb_{1,k}$ are both fixed as $\mathbf{0}$; $\vb_{k}$ is fixed as $\mathbf{1}$; $\vp_{k}$ is fixed as $(1/L)\mathbf{1}$. The results are reported in Figure~\ref{fig:lasso-ablation}.

\begin{figure}
 \includegraphics[width=\linewidth]{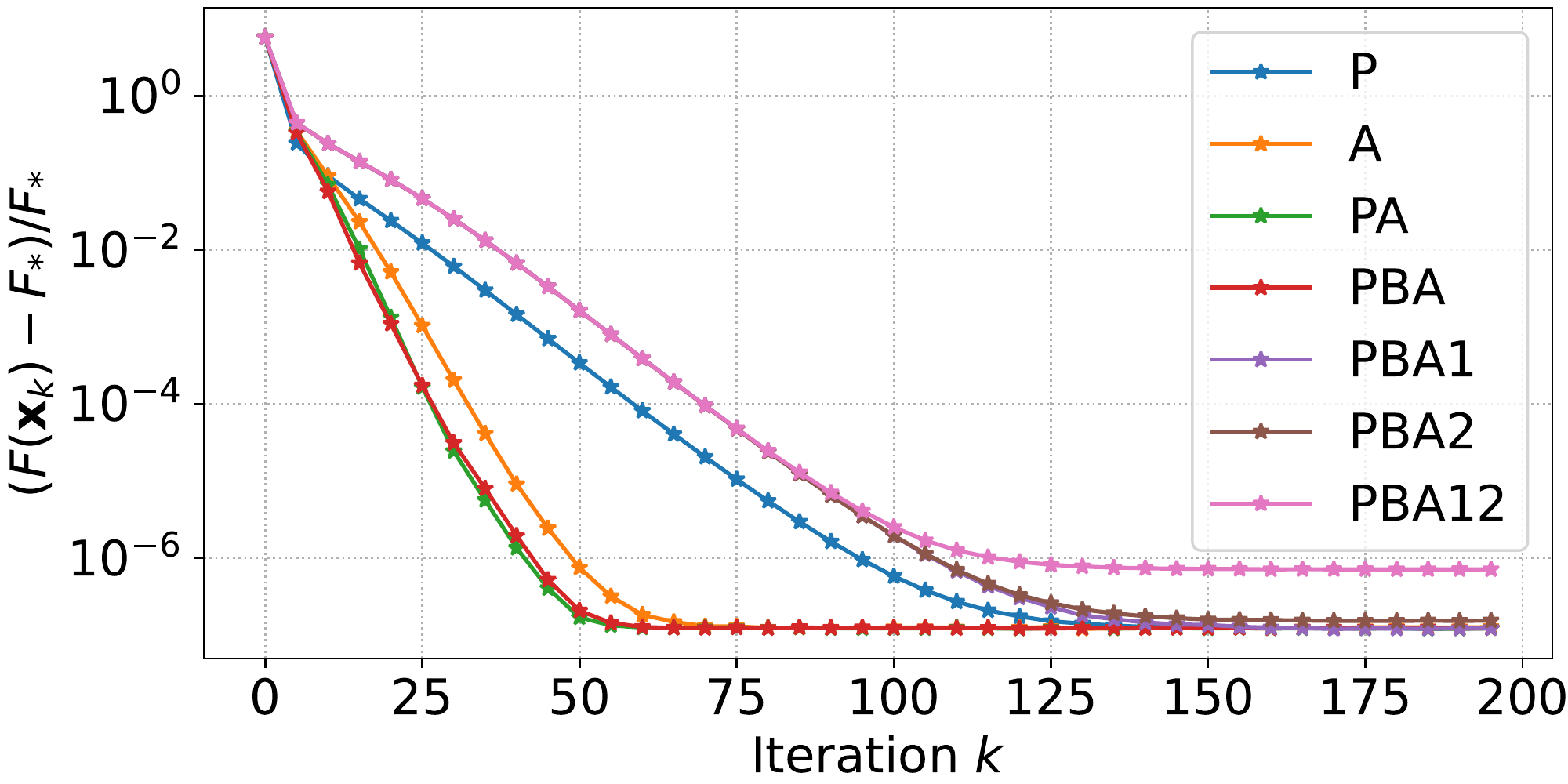}
  \caption{\label{fig:lasso-ablation}Ablation study on LASSO.} 
\end{figure}
\begin{figure}
  \includegraphics[width=\linewidth]{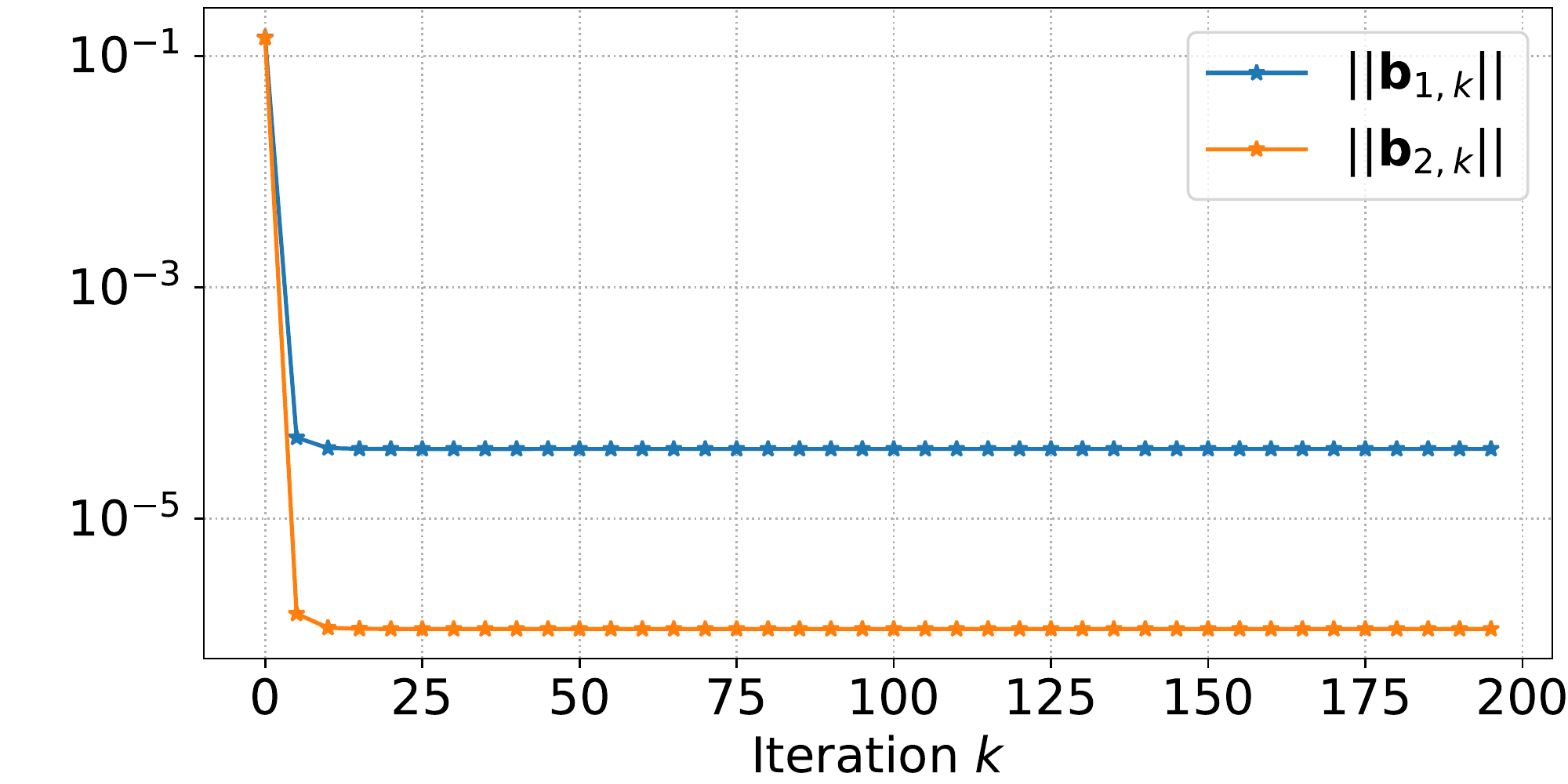}
  \caption{\label{fig:lasso-b12}Visualization of the learned $\vb_{1,k},\vb_{2,k}$.}
\end{figure}

In Figure~\ref{fig:lasso-ablation}, $F_{\ast}$ denotes the best objective value of each instance and $(F(\vx_{k})-F_{\ast})/F_{\ast}$ measures the average optimality gap on the test set. Each curve describes the convergence performance of a learned optimizer. From Figure~\ref{fig:lasso-ablation}, one can conclude that, with all of the learned optimizers, convergence can be reached within the machine precision. 
 
An interesting observation is that the proposed scheme \eqref{eq:final-scheme} may not benefit from parameterizing and learning more components. Comparing PBA, PBA1, PBA2 and PBA12, we find that fixing $\vb_{1,k}=\vb_{2,k}=\mathbf{0}$ is a better choice than learning them. This phenomenon can be explained by our theory. In Theorem~\ref{theo:nesterov}, both $\vb_{2,k}$ and $\vb_{1,k}$ are expected to converge to zero, otherwise, the convergence would be violated. 
However, in Figure~\ref{fig:lasso-b12} where we report the average values of $\|\vb_{1,k}\|$ and $\|\vb_{2,k}\|$ in PBA12, both $\|\vb_{1,k}\|$ and $\|\vb_{2,k}\|$ converge to a relatively small value after a few ($\leq 10$) iterations, but there is no guarantee that they converge exactly to zero given parameterization \eqref{eq:final-scheme-lstm}. Such observation actually validates our theories and suggests to fix $\vb_{2,k}$ and $\vb_{1,k}$ as $\mathbf{0}$ instead of learning them.
 
\paragraph{A Simplified Scheme.} Furthermore, comparing PA and PBA in Figure~\ref{fig:lasso-ablation}, we find that parameterizing $\vb_{k}$ does not show significant benefits. To reduce computational overhead, we adopt PA and fix $\vb_{1,k}=\vb_{2,k}=\mathbf{0}$ and $\vb_{k}=\mathbf{1}$, which reduces \eqref{eq:final-scheme} and \eqref{eq:final-scheme-lstm} to the following simplified scheme:
\begin{equation}
    \label{eq:l2o-pa}\tag{L2O-PA}
    \begin{aligned}
    \vo_{k},\vh_{k} &= \mathrm{LSTM}\big(\vx_{k},\nabla f(\vx_{k}),\vh_{k-1}; \phi_\text{LSTM}\big),\\
    \vp_{k},\va_{k} &= \mathrm{MLP}(\vo_{k}; \phi_{\text{MLP}}),\\
    \vx_{k+1} &= \prox_{r, \vp_{k}} \big( \vy_{k} - \vp_{k}\nabla f(\vy_{k}) \big),\\
    \vy_{k+1} &= \vx_{k+1} + \va_{k} \odot (\vx_{k+1} - \vx_{k}).
    \end{aligned}
\end{equation}

\subsection{Comparison with Competitors}

We compare \eqref{eq:l2o-pa} with some competitors in two settings. In the first setting (also coined as In-Distribution), the training and testing samples are generated independently with the same distribution. The second setting is much more challenging, where the models are trained on randomly synthetic data but tested on real data, which is coined as Out-of-Distribution or OOD.

\paragraph{Solving LASSO Regression.} In-Distribution experiments follow the settings in Section~\ref{sec:abla}, where the training and testing sets are generated randomly and independently, but share the same distribution. In contrast, Out-of-Distribution (OOD) experiments also generate training samples similarly to In-Distribution experiments, but the test samples are generated based on a natural image benchmark BSDS500~\citep{MartinFTM01}. Detailed information about the synthetic and real data used can be found in the Appendix.

\begin{figure}
    \centering
    \includegraphics[width=\linewidth]{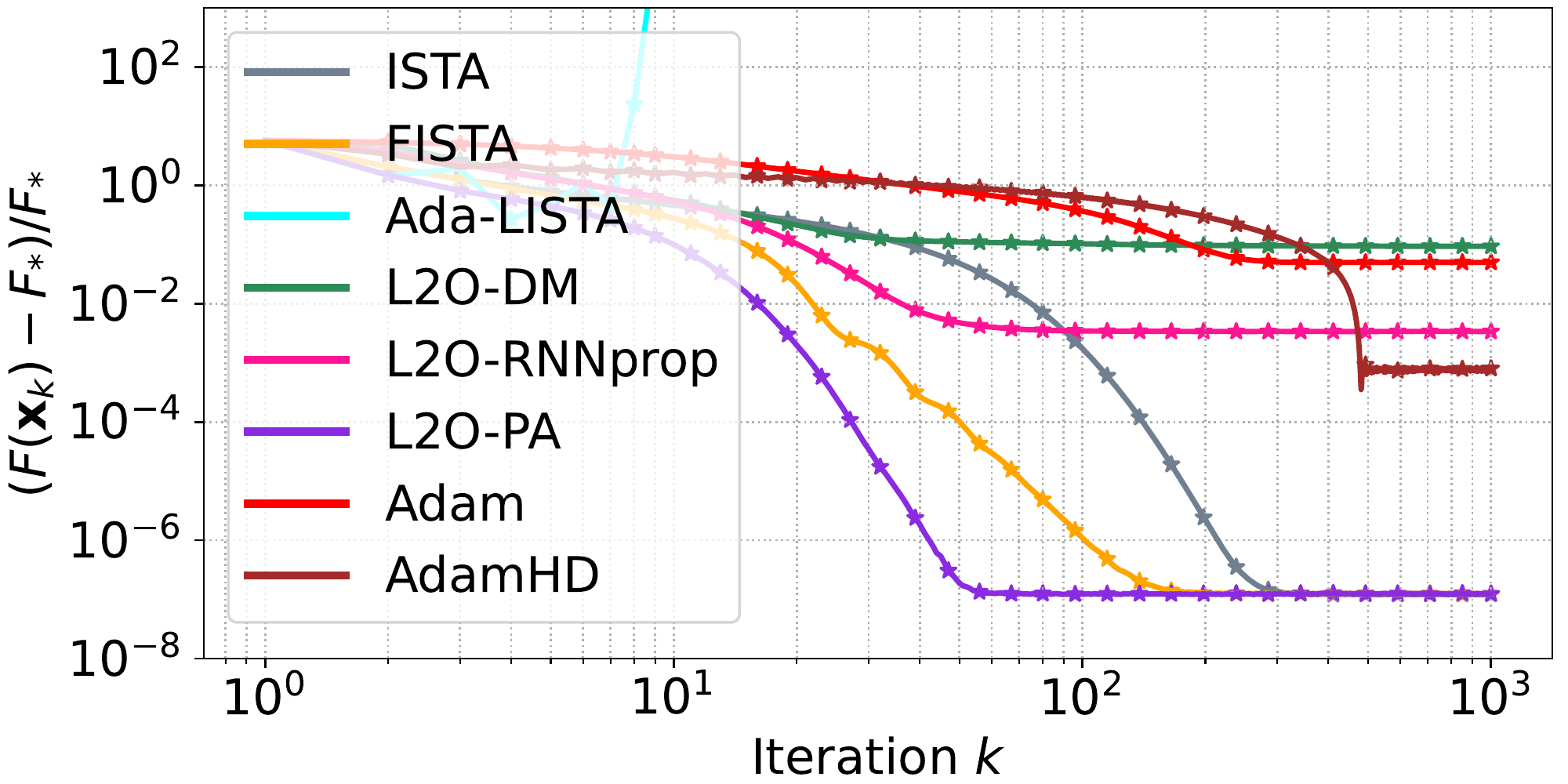}
    \caption{LASSO: Train and test on synthetic data.}
    \label{fig:lasso-benchmark}
\end{figure}

\begin{figure}
    \centering
    \includegraphics[width=\linewidth]{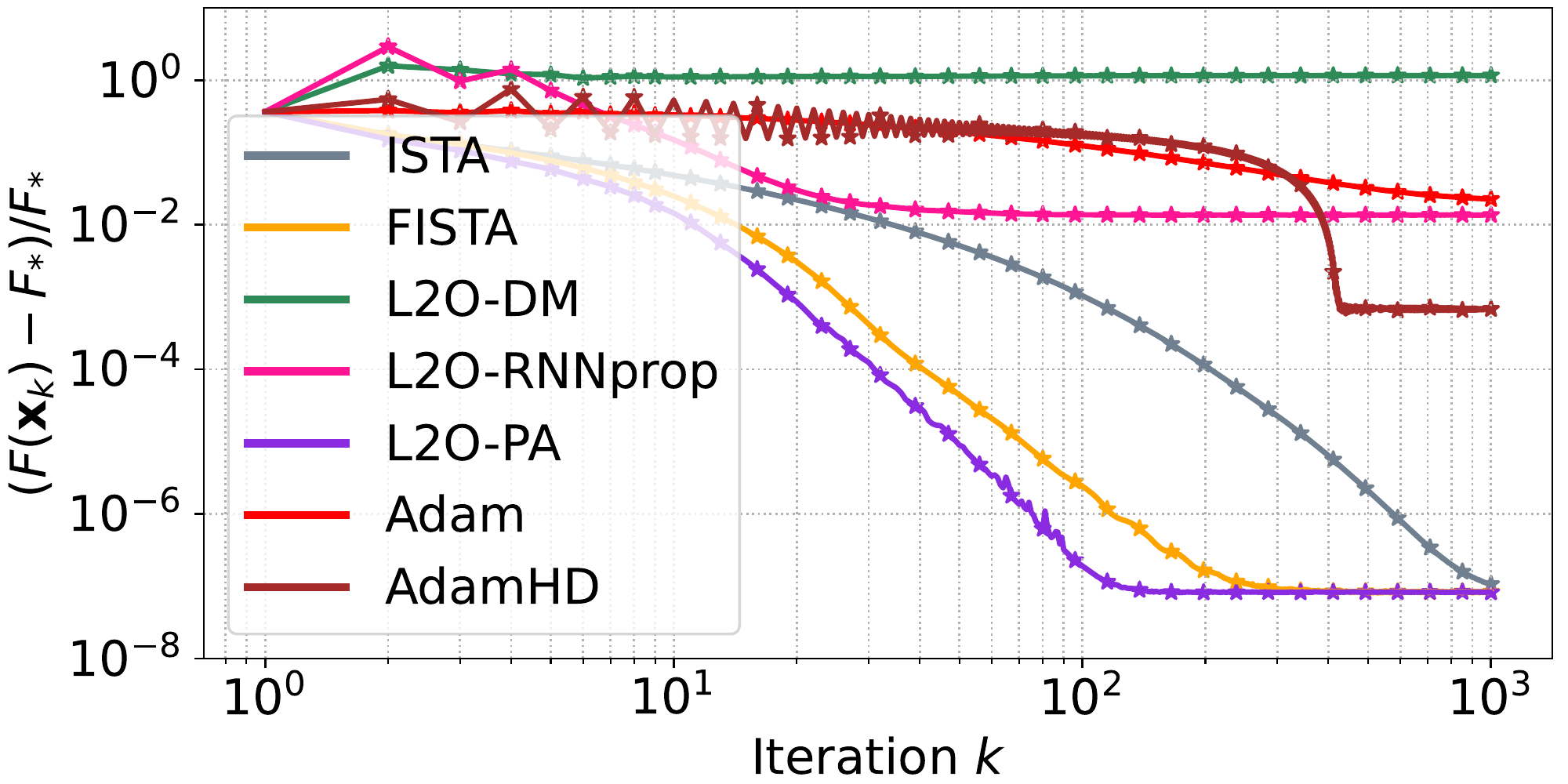}
    \caption{LASSO: Train on synthetic data and test on real data.}
    \label{fig:lasso-real}
\end{figure}

We first choose ISTA, FISTA~\citep{beck2009fast}, and Adam~\citep{kingma2014adam} as baselines since they are classical manually-designed update rules. We choose a state-of-the-art Algorithm-Unrolling method Ada-LISTA~\citep{aberdam2021ada} as another baseline since it is applicable in the settings of various $\vA$. Additionally, we choose two generic L2O methods that following \eqref{eq:rule-z} as baselines: L2O-DM \citep{andrychowicz2016learning} and L2O-RNNprop \citep{lv2017learning}. Finally, we choose AdamHD~\citep{baydin2017online}, a hyperparameter optimization (HPO) method that adaptively tunes the learning rate in Adam with online learning, as a baseline. Since Ada-LISTA is not applicable to problems that have different sizes with training samples, we only compare our method with Ada-LISTA in In-Distribution experiments.

In-Distribution results are reported in Figure~\ref{fig:lasso-benchmark} and Out-of-Distribution results are reported in Figure~\ref{fig:lasso-real}. In both of the settings, the proposed \eqref{eq:l2o-pa} performs competitively. In the OOD setting, the other two learning-based methods, L2O-DM and L2O-RNNprop, both struggle to converge with optimality tolerance $10^{-2}$, but \eqref{eq:l2o-pa} is still able to converge until touching the machine precision.

\paragraph{Solving $\ell_1$-regularized Logistic Regression.}
An $\ell_1$-regularized logistic regression for binary classification is characterized by set of training examples $\{(\va_i, b_i) \in \mathbb{R}^n\times\{0,1\}\}_{i=1}^m$ where $\va_i$ is an $n$-dimensional feature vector and $b_i$ is a binary label. To solve the $\ell_1$-regularized logistic regression problem is to find an optimal $\vx_\ast$ so that $h(\va_i^\top\vx_\ast)$ predicts $p(b_i|\va_i)$ well, where $h(c)=1/(1+e^{-c})$ is the logistic function. The exact formula of the objective function can be found in the Appendix.

We train L2O models on randomly generated logistic regression datasets for binary classification. Each dataset contains 1,000 samples that have 50 features. We evaluate the trained models in both the In-Distribution and OOD settings. Results are shown in Figures~\ref{fig:logistic-benchmark} and \ref{fig:logistic-real} respectively. Details about synthetic data generation and real-world datasets are provided in the Appendix.

Under the In-Distribution setting (Figure~\ref{fig:logistic-benchmark}), our method, i.e., L2O-PA, can converge to optimal solutions within 10 iterations, almost 100$\times$ faster than FISTA, while the other two generic L2O baselines fail to converge. When tested on OOD optimization problems (Figure~\ref{fig:logistic-real}), L2O-PA can still converge within 100 iterations, faster than FISTA by more than 10 times.


\begin{figure}
    \centering
    \includegraphics[width=\linewidth]{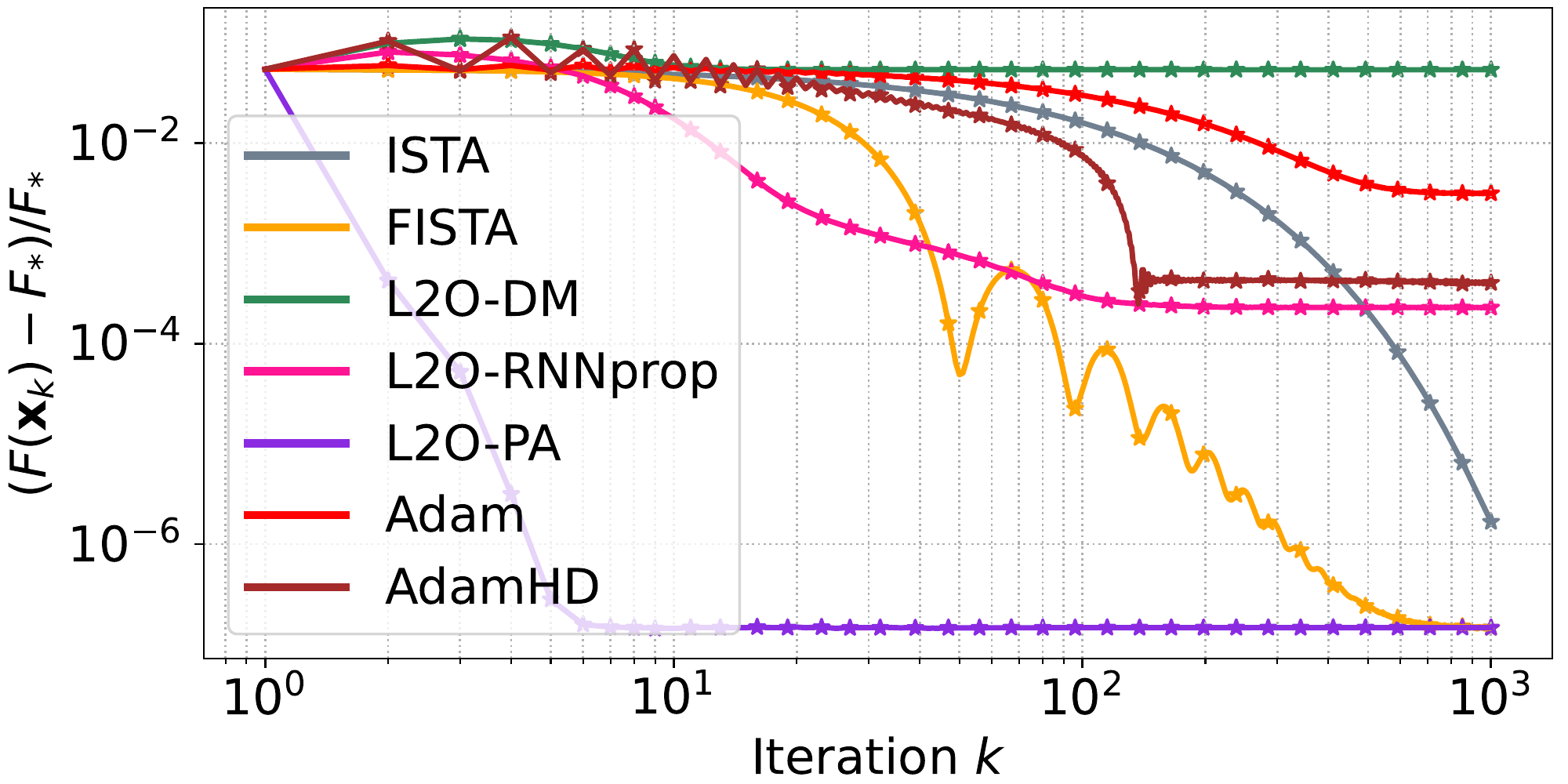}
    \caption{Logistic: Train and test on synthetic data.}
    \label{fig:logistic-benchmark}
\end{figure}

\begin{figure}
    \centering
    \includegraphics[width=\linewidth]{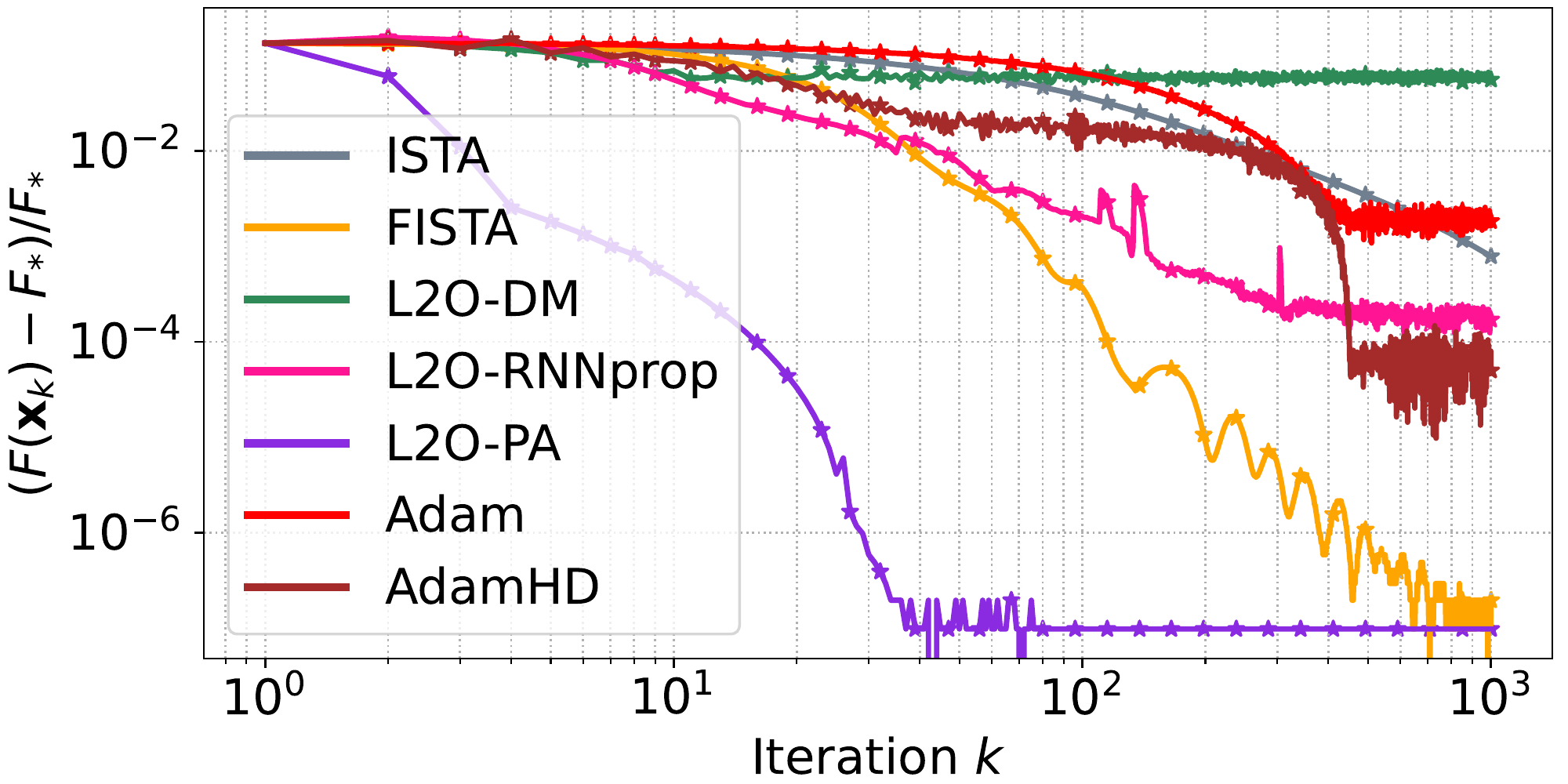}
    \caption{Logistic: Train on synthetic data and test on real data.}
    \label{fig:logistic-real}
\end{figure}

\section{Conclusions and Future Work}
\label{sec:conclusion}
By establishing the basic conditions of the update rule, we incorporate mathematical structures into machine-learning-based optimization algorithms (learned optimizers). In our settings, we do not learn the entire update rule as a black box. Instead, we propose to learn a preconditioner and an accelerator, and then construct the update rule in the style of proximal gradient descent. Our approach is applicable to a broad class of optimization problems while providing superior empirical performance. This study actually provides some insights toward answering an important question in L2O: Which part of the model should be mathematically grounded and which part could be learned?

There are several lines of future work. Firstly, relaxing the assumption of convexity and studying the nonconvex settings will be a significant future direction. 
Another interesting topic is to extend \eqref{eq:l2o-nesterov-y} and consider update rules that adopt more input information and longer horizons.

\section*{Acknowledgements}
The work of HanQin Cai was partially supported by NSF DMS 2304489.


\bibliographystyle{icml2023}
\bibliography{ref}

\newpage
\appendix
\onecolumn

\section{Proof of Theorems}
\label{sec:a}

\subsection{Preliminaries}

In our proofs, 
$\|\cA\|$ denotes the spectral norm of matrix $\cA$, $\|\cA\|_{\mathrm{F}}$ denotes the Frobenius norm of matrix $\cA$, $\|\vx\|$ denotes the $\ell_2$-norm of vector $\vx$, and $\|\vx\|_1$ denotes the $\ell_1$-norm of vector $\vx$.

Before the proofs of those theorems in the main text, we describe a lemma here to facilitate our proofs.

\begin{lemma}\label{lemma:mvt}
For any operator $\vo \in \cD_{C}(\mathbb{R}^{m\times n})$ and any $\vx_{1},\vy_{1},\vx_{2},\vy_{2},\cdots,\vx_{m},\vy_{m} \in \mathbb{R}^{n}$, there exist matrices $\vJ_{1},\vJ_{2},\cdots,\vJ_{m} \in \mathbb{R}^{n\times n}$ such that
\begin{equation}
    \label{eq:lemma-mvt}
    \vo(\vx_{1},\vx_{2},\cdots,\vx_{m}) - \vo(\vy_{1},\vy_{2},\cdots,\vy_{m}) = \sum_{j=1}^{m} \vJ_{j}(\vx_{j}-\vy_{j}),
\end{equation}
and
\begin{equation}
    \label{eq:lemma-j-bound}
    \|\vJ_{1}\| \leq \sqrt{n}C,\quad \|\vJ_{2}\| \leq \sqrt{n}C, \quad \cdots, \quad\|\vJ_{m}\| \leq \sqrt{n}C.
\end{equation}
\end{lemma}
\begin{proof}
Since $\vo \in \cD_{C}(\mathbb{R}^{m\times n})$, the outcome of operator $\vo$ is an $n$-dimensional vector. Now we denote the $i$-th element as $o_{i}(1\leq i \leq n)$ and write operator $\vo$ in a matrix form:
\[
\begin{aligned}
 \vo(\vx_{1},\vx_{2},\cdots,\vx_{m}) &=  \Big[  o_{1}(\vx_{1},\vx_{2},\cdots,\vx_{m}), ~\cdots,~ o_{n}(\vx_{1},\vx_{2},\cdots,\vx_{m})  \Big]^\top, \\
  \vo(\vy_{1},\vy_{2},\cdots,\vy_{m}) &= \Big[  o_{1}(\vy_{1},\vy_{2},\cdots,\vy_{m}), ~\cdots,~ o_{n}(\vy_{1},\vy_{2},\cdots,\vy_{m})  \Big]^\top .
\end{aligned}
\]
Applying the Mean Value Theorem on $o_{i}(1\leq i \leq n)$, one could obtain
\begin{equation}
    \label{eq:app-mvt-single}
    \begin{aligned}
 &~ o_{i}(\vx_{1},\vx_{2},\cdots,\vx_{m}) - o_{i}(\vy_{1},\vy_{2},\cdots,\vy_{m})  \\
 = &~ \sum_{j=1}^{m} \left\langle  \frac{\partial o_{i}}{\partial \vx_{j}} \Big( \xi_{i}\vx_{1} + (1 - \xi_{i}) \vy_{1}, \cdots , \xi_{i}\vx_{m} + (1 - \xi_{i}) \vy_{m} \Big) , \vx_{j} - \vy_{j}  \right\rangle, 
\end{aligned}
\end{equation}
for some $\xi_{i} \in (0,1)$. For simplicity, we denote
\[\vz_{i} := \Big[ \xi_{i}\vx_{1} + (1 - \xi_{i}) \vy_{1}, \xi_{i}\vx_{2} + (1 - \xi_{i}) \vy_{2}, \cdots, \xi_{i}\vx_{m} + (1 - \xi_{i}) \vy_{m} \Big]^\top,\quad \forall 1 \leq i \leq n,\]
and stack all partial derivatives into one matrix as
\[\vJ_{j} = \left[  \frac{\partial o_{1}}{\partial \vx_{j}}\big(\vz_{1}\big), ~~\frac{\partial o_{2}}{\partial \vx_{j}}\big(\vz_{2}\big), ~~\cdots, ~~\frac{\partial o_{n}}{\partial \vx_{j}}\big(\vz_{n}\big)  \right]^\top \in \mathbb{R}^{n \times n}.\]
Then one can immediately get \eqref{eq:lemma-mvt} from \eqref{eq:app-mvt-single}. The upper bound of $\|\vJ_{j}\|(1\leq j \leq m)$ can be estimated by
\[\|\vJ_{j}\|^2 \leq \|\vJ_{j}\|_{\mathrm{F}}^2 = \sum_{i=1}^{n} \left\| \frac{\partial o_{i}}{\partial \vx_{j}}\big(\vz_{i}\big) \right\|^2 \leq n C^2.\]
Therefore, it concludes that $\|\vJ_{j}\| \leq \sqrt{n}C$ for all $1 \leq j \leq m$, which finishes the proof. 

Note that such upper bound of $\|\vJ_{j}\|$ is tight. We could NOT directly conclude $\|\vJ_{j}\| \leq C$ because $\partial o_{1}/\partial \vx_{j}, \cdots, \partial o_{n}/\partial \vx_{j}$ are evaluated respectively at points $\vz_{1},\cdots,\vz_{n}$, and, consequently, the whole matrix $\vJ_{j}$ is not a Jacobian matrix of operator $\vo$.
\end{proof}

\subsection{Proof of Theorem~\ref{theo:gd}}
 \begin{proof}
Denote
 \[
 \hat{\vd}_{k} := \vd_{k}(\vx_{\ast},\mathbf{0}).
 \]
 Plugging the above equation into \eqref{eq:l2o-gd}, we obtain
 \[\vx_{k+1} = \vx_{k} - \vd_{k}(\vx_{k},\nabla f(\vx_{k})) + \vd_{k}(\vx_{\ast},\mathbf{0}) -  \hat{\vd}_{k}.\]
 Applying Lemma~\ref{lemma:mvt}, we have
 \[\vx_{k+1} = \vx_{k} - \vJ_{1,k}(\vx_{k}-\vx_{\ast}) - \vJ_{2,k}\nabla f(\vx_{k}) -  \hat{\vd}_{k},\]
 for some $\vJ_{1,k},\vJ_{2,k} \in \mathbb{R}^{n\times n}$ that satisfy
 \[\|\vJ_{1,k}\| \leq \sqrt{n}C, \quad \|\vJ_{2,k}\| \leq \sqrt{n}C.\]
  Define 
 \[
 \vP_{k} :=\vJ_{2,k},\quad
 \vb_{k} :=\vJ_{1,k}(\vx_{k}-\vx_{\ast}) + \hat{\vd}_{k}.
 \]
 Then we obtain $\vd_{k}(\vx_{k},\nabla f(\vx_{k})) = \vP_{k}\nabla f(\vx_{k}) + \vb_{k}$ and it holds that
 \[
 \begin{aligned}
 \big\| \vP_{k} \big\| &\leq \sqrt{n} C,\\
 \big\| \vb_{k} \big\| &\leq \sqrt{n} C \big\| \vx_{k} -\vx_{\ast}\big\| + \big\| \hat{\vd}_{k} \big\|.
 \end{aligned}
 \]
Assumption \eqref{eq:fp1} leads to $\| \hat{\vd}_{k} \| \to 0$ and Assumption \eqref{eq:gc1} leads to $\| \vx_{k} -\vx_{\ast} \|\to 0$. Consequently, $\| \vb_{k} \| \to 0$, which finishes the proof.
 \end{proof}

\subsection{Proof of Theorem~\ref{theo:ppa}}
\begin{proof}
Following the same proof line with that of Theorem~\ref{theo:gd}, we first denote $\hat{\vd}_{k} := \vd_{k}(\vx_{\ast},\mathbf{0})$ and then obtain 
\[
\begin{aligned}
 \vx_{k+1} &= \vx_{k} - \vd_{k}(\vx_{k+1},\vg_{k+1}) + \vd_{k}(\vx_{\ast},\mathbf{0}) -  \hat{\vd}_{k}\\
 &= \vx_{k} - \vJ_{1,k}(\vx_{k+1}-\vx_{\ast}) - \vJ_{2,k}\vg_{k+1} -  \hat{\vd}_{k},\\
 &= \vx_{k} - \big( \underbrace{\vJ_{2,k}}_{\vP_{k}}\vg_{k+1} + \underbrace{\vJ_{1,k}(\vx_{k+1}-\vx_{\ast}) + \hat{\vd}_{k}}_{\vb_{k}} \big),
\end{aligned}
\]
where $\vg_{k+1} \in \partial r(\vx_{k+1})$, $\vP_{k}$ is bounded, and $\vb_{k}\to \mathbf{0}$ as $k\to\infty$. In another word, $\vx_{k+1}$ satisfies
\begin{equation}
\label{eq:app-ppa}
    \vx_{k} - \vb_{k} \in \vx_{k+1} + \vP_{k}\partial r(\vx_{k+1}).
\end{equation}
Note that $\vP_{k}$ and $\vb_{k}$ may depend on $\vx_{k+1}$, but it does not hurt our conclusion: For any operator sequence $\{\vd_{k}\}$ that satisfies \eqref{eq:fp2} and any sequence $\{\vx_{k}\}$ that is generated by \eqref{eq:l2o-ppa} and satisfies the Global Convergence, there must exist $\{\vP_{k},\vb_{k}\}$ such that \eqref{eq:app-ppa} holds for all $k$.

Meanwhile, since $\vP_{k}$ is assumed to be symmetric positive definite, function $\hat{F}_{k}(\vx) = f(\vx) + (1/2)\|\vx - \vx_{k} + \vb_{k}\|^2_{\vP_{k}^{-1}}$ is strongly convex. Therefore, $\vx$ is the unique minimizer of $\hat{F}_{k}$ if and only if:
\[\mathbf{0} \in \partial r(\vx) + \vP_{k}^{-1}(\vx - \vx_{k} + \vb_{k}).\]
With reorganization, the above condition is equivalent with
\[\vx_{k} - \vb_{k} \in \vx + \vP_{k}\partial r(\vx),\]which is exactly the condition \eqref{eq:app-ppa} that $\vx_{k+1}$ satisfies. Thus, $\vx_{k+1}$ is the unique minimizer of $\hat{F}_{k}(\vx)$ and is the unique point satisfying \eqref{eq:app-ppa}, which finishes the proof.
\end{proof}

\subsection{Proof of Theorem~\ref{theo:fbs}}

\begin{proof}
 Denote
 \[
 \hat{\vd}_{k} := \vd_{k}(\vx_{\ast},\nabla f(\vx_{\ast}), \vx_{\ast}, -\nabla f(\vx_{\ast})) .
 \]
 Plugging the above equation into \eqref{eq:l2o-fbs}, we obtain
 \[\vx_{k+1} = \vx_{k} - \vd_{k}(\vx_{k},\nabla f(\vx_{k}),\vx_{k+1},\vg_{k+1}) + \vd_{k}(\vx_{\ast},\nabla f(\vx_{\ast}), \vx_{\ast}, -\nabla f(\vx_{\ast})) - \hat{\vd}_{k}.\]
Applying Lemma~\ref{lemma:mvt}, we obtain
 \[
 \begin{aligned}
  \vx_{k+1} = ~\vx_{k} & - \vJ_{1,k}(\vx_{k}-\vx_{\ast}) - \vJ_{2,k}(\vx_{k+1}-\vx_{\ast}) \\
  & - \vJ_{3,k}(\nabla f(\vx_{k}) - \nabla f(\vx_{\ast})) - \vJ_{4,k}(\vg_{k+1} + \nabla f(\vx_{\ast})) - \hat{\vd}_{k}
 \end{aligned}
 \]
 for some $\vJ_{1,k},\vJ_{2,k},\vJ_{3,k},\vJ_{4,k} \in \mathbb{R}^{n\times n}$ that satisfy
 \[\|\vJ_{j,k}\| \leq \sqrt{n}C, \quad \forall j = 1,2,3,4.\]
 Reorganizing the above equation, we have
 \[
 \begin{aligned}
  \vx_{k+1} = ~ \vx_{k} & - \vJ_{1,k}(\vx_{k}-\vx_{\ast}) - \vJ_{2,k}(\vx_{k+1}-\vx_{\ast}) \\
  & - (\vJ_{3,k}-\vJ_{4,k})(\nabla f(\vx_{k}) - \nabla f(\vx_{\ast})) - \vJ_{4,k}(\vg_{k+1} + \nabla f(\vx_{k}) ) - \hat{\vd}_{k}.
 \end{aligned}
 \]
 With 
 \[
 \begin{aligned}
  \vP_{k} &:= \vJ_{4,k},\\
  \vb_{k} &:= \vJ_{1,k}(\vx_{k}-\vx_{\ast}) + \vJ_{2,k}(\vx_{k+1}-\vx_{\ast}) + (\vJ_{3,k}-\vJ_{4,k})(\nabla f(\vx_{k}) - \nabla f(\vx_{\ast})) + \hat{\vd}_{k},
 \end{aligned}
 \]
 we have 
 \[\vx_{k+1} = \vx_{k} - \vP_{k}(\nabla f(\vx_{k}) + \vg_{k+1}) - \vb_{k},\quad \vg_{k+1} \in \partial r(\vx_{k+1}),\]
 and
 \[
 \begin{aligned}
  \big\| \vP_{k} \big\| &\leq  \sqrt{n} C,\\
 \big\| \vb_{k} \big\| &\leq \sqrt{n} C \big\| \vx_{k} -\vx_{\ast}\big\| + \sqrt{n} C \big\| \vx_{k+1} -\vx_{\ast}\big\| + 2\sqrt{n}C \big\| \nabla f(\vx_{k}) - \nabla f(\vx_{\ast}) \big\| + \big\| \hat{\vd}_{k} \big\|.
 \end{aligned}
 \]
 The smoothness of $f$ implies $\nabla f(\vx_{k}) - \nabla f(\vx_{\ast}) \to \mathbf{0}$. Consequently, we conclude with $ \vb_{k} \to \mathbf{0}$ and this finishes the proof of the first part in Theorem~\ref{theo:fbs}.
 

Now it is enough to prove that the following inclusion equation of $\vx$ has a unique solution and is equivalent with \eqref{eq:l2o-fbs-quadratic}.
\begin{equation}
    \label{eq:app-fbs-inclusion}
    \vx = \vx_{k} - \vP_{k}(\nabla f(\vx_{k}) + \vg) - \vb_{k},\quad \vg \in \partial r(\vx).
\end{equation}
 The above equation is equivalent with
 \[\vx \in \vx_{k} - \vP_{k}(\nabla f(\vx_{k}) + \partial r(\vx)) - \vb_{k}.\]
 Since $\vP_{k}$ is assumed to be symmetric positive definite, one could obtain another equivalent form with reorganization:
 \[\mathbf{0} \in \partial r(\vx) + \vP_{k}^{-1}(\vx - \vx_{k} + \vP_{k}\nabla f(\vx_{k}) + \vb_{k}).\]
 Thanks to $f \in \cF_{L}(\mathbb{R}^n)$ and $r \in \cF(\mathbb{R}^n)$, the above equation has an unique solution $\vx^{+}$ that yields
 \[\vx^{+} = \argmin_{\vx} r(\vx) + \frac{1}{2}\left\|\vx - \vx_{k} + \vP_{k}\nabla f(\vx_{k}) + \vb_{k}\right\|^2_{\vP_{k}^{-1}} = \mathrm{prox}_{r,\vP_{k}}\left(\vx_{k} - \vP_{k}\nabla f(\vx_{k}) - \vb_{k}\right).\]
 Since $\vx_{k+1}$ satisfies \eqref{eq:app-fbs-inclusion}, we conclude that $\vx_{k+1} = \vx^{+}$ and finish the whole proof.
 \end{proof}

\subsection{Proof of Theorem~\ref{theo:nesterov}}
\paragraph{Step 1: Analyzing (\ref{eq:l2o-nesterov-x}).} Denote \[\hat{\vd}_{k}:=\vd_{k}(\vx_{\ast},\nabla f(\vx_{\ast}),\vx_{\ast},-\nabla f(\vx_{\ast}),\vx_{\ast},\nabla f(\vx_{\ast})).\]
Then \eqref{eq:l2o-nesterov-x} can be written as
\[
\begin{aligned}
 \vx_{k+1} = \vx_{k} & - \vd_{k}(\vx_{k},\nabla f(\vx_{k}),\vx_{k+1},\vg_{k+1},\vy_{k},\nabla f(\vy_{k}))\\
& + \vd_{k}(\vx_{\ast},\nabla f(\vx_{\ast}),\vx_{\ast},-\nabla f(\vx_{\ast}),\vx_{\ast},\nabla f(\vx_{\ast})) - \hat{\vd}_{k},
\end{aligned}
\]
where $\vg_{k+1} \in \partial r(\vx_{k+1})$. Applying Lemma~\ref{lemma:mvt}, we have
\[
    \begin{aligned}
 \vx_{k+1} = \vx_{k} & - \vJ_{1,k}(\vx_{k}-\vx_\ast) - \vJ_{2,k}(\vx_{k+1}-\vx_{\ast}) - \vJ_{3,k}(\vy_{k}-\vx_{\ast})- \hat{\vd}_{k}\\
& - \vJ_{4,k}(\nabla f(\vx_{k}) - \nabla f(\vx_{\ast}) ) - \vJ_{5,k}(\vg_{k+1} + \nabla f(\vx_{\ast})) - \vJ_{6,k}(\nabla f(\vy_{k}) - \nabla f(\vx_{\ast})),
\end{aligned}
\]
where matrices $\vJ_{j,k}(1 \leq j \leq 6)$ satisfy
\[\|\vJ_{j,k}\| \leq \sqrt{n}C, \quad \forall j=1,2,3,4,5,6.\]
Then we do some calculations and get
\[\begin{aligned}
 \vx_{k+1} = \vx_{k} & - \vJ_{1,k}(\vx_{k}-\vx_\ast) - \vJ_{2,k}(\vx_{k+1}-\vx_{\ast}) - \vJ_{3,k}(\vy_{k}-\vx_{\ast})- \hat{\vd}_{k}\\
& - (\vJ_{4,k} - \vJ_{5,k} + \vJ_{6,k} )(\nabla f(\vx_{k}) - \nabla f(\vx_{\ast}) ) - (\vJ_{5,k} - \vJ_{6,k})(\nabla f(\vx_{k}) - \nabla f(\vx_{\ast}) )\\
& - \vJ_{5,k}(\vg_{k+1} + \nabla f(\vx_{\ast})) - \vJ_{6,k}(\nabla f(\vy_{k}) - \nabla f(\vx_{\ast}))\\
 = \vx_{k} & - \vJ_{1,k}(\vx_{k}-\vx_\ast) - \vJ_{2,k}(\vx_{k+1}-\vx_{\ast}) - \vJ_{3,k}(\vy_{k}-\vx_{\ast})- \hat{\vd}_{k}\\
& - (\vJ_{4,k} - \vJ_{5,k} + \vJ_{6,k} )(\nabla f(\vx_{k}) - \nabla f(\vx_{\ast}) ) \\
& - (\vJ_{5,k} - \vJ_{6,k}) \nabla f(\vx_{k}) - \vJ_{5,k}~ \vg_{k+1} - \vJ_{6,k} \nabla f(\vy_{k}).
\end{aligned}\]
Given any $\vB_{k} \in \mathbb{R}^{n \times n}$, we define:
\[
\begin{aligned}
 \vP_{1,k} &:= \vJ_{5,k},\\
 \vP_{2,k} &:= \vJ_{6,k},\\
 \vb_{1,k} &:= \vJ_{1,k}(\vx_{k}-\vx_\ast) + \vJ_{2,k}(\vx_{k+1}-\vx_{\ast}) + \vJ_{3,k}(\vy_{k}-\vx_{\ast}) + \hat{\vd}_{k}\\
 &\quad + (\vJ_{4,k} - \vJ_{5,k} + \vJ_{6,k} )(\nabla f(\vx_{k}) - \nabla f(\vx_{\ast}) ) + \vB_{k}(\vy_{k} - \vx_{k}). 
\end{aligned}
\]
Then we have
\[\vx_{k+1} = ~ \vx_{k} - (\vP_{1,k}-\vP_{2,k})\nabla f(\vx_{k}) - \vP_{2,k}\nabla f(\vy_{k}) - \vP_{1,k} ~\vg_{k+1} + \vB_{k}(\vy_{k} - \vx_{k}) - \vb_{1,k},\]
which immediately leads to \eqref{eq:l2o-nesterov-d}. The upper bounds of $\vJ_{j,k}(1 \leq j \leq 6)$ imply that $\vP_{1,k},\vP_{2,k}$ are bounded:
\[\|\vP_{1,k}\| \leq \sqrt{n}C,\quad \|\vP_{2,k}\| \leq \sqrt{n}C,\]
and $\vb_{1,k}$ is controlled by
\[
\begin{aligned}
 \|\vb_{1,k}\| \leq &~ \sqrt{n}C \Big(\|\vx_{k}-\vx_{\ast}\| + \|\vx_{k+1}-\vx_{\ast}\| + \|\vy_{k}-\vx_{\ast}\| \Big) + \|\hat{\vd}_{k}\|\\ & + 3\sqrt{n}C\|\nabla f(\vx_{k}) - \nabla f(\vx_{\ast})\| + \|\vB_{k}\|\|\vy_{k}-\vx_{k}\|.
\end{aligned}
\]
Assumption \eqref{eq:gc4} implies that 
\[\|\vx_{k}-\vx_{\ast}\| \to 0,~~ \|\vx_{k+1}-\vx_{\ast}\|\to 0,~~ \|\vy_{k}-\vx_{\ast}\|\to 0,~~\|\vy_{k}-\vx_{k}\| \to 0,\]and Assumption \eqref{eq:fp4} implies that $\|\hat{\vd}_{k}\| \to 0$.
The smoothness of $f$ implies $\|\nabla f(\vx_{k}) - \nabla f(\vx_{\ast}) \| \to 0$. In the theorem statement, we assume that $\{\vB_{k}\}$ could be any bounded matrix sequence. Therefore, it concludes that $\|\vb_{1,k}\| \to 0$ as $k \to \infty$.

\paragraph{Step 2: Analyzing (\ref{eq:l2o-nesterov-y}).} Since $T=1$, equation \eqref{eq:l2o-nesterov-y} reduces to \[\vy_{k+1} = \vm(\vx_{k+1},\vx_{k}).\] 
Due to Assumption \eqref{eq:fp4}, $\vx_{\ast} = \vm(\vx_{\ast},\vx_{\ast})$, equation \eqref{eq:l2o-nesterov-y} is equivalent to 
\[\vy_{k+1} = \vm(\vx_{k+1},\vx_{k}) - \vm(\vx_{\ast},\vx_{\ast}) + \vx_{\ast}.\]
Then one could apply Lemma~\ref{lemma:mvt} and obtain
\[
    \vy_{k+1} = \vJ_{7,k}(\vx_{k+1}-\vx_{\ast}) + \vJ_{8,k}(\vx_{k}-\vx_{\ast}) + \vx_{\ast},
\]
where matrices $\vJ_{7,k}$ and $\vJ_{8,k}$ satisfy \[\|\vJ_{7,k}\| \leq \sqrt{n}C, \quad \|\vJ_{8,k}\| \leq \sqrt{n}C.\]
With calculation, we get
\[
\begin{aligned}
 \vy_{k+1} &= \vJ_{7,k}\vx_{k+1} + \vJ_{8,k}\vx_{k} + (\vI - \vJ_{7,k} - \vJ_{8,k})\vx_{\ast}\\
 &= (\vI - \vJ_{8,k})\vx_{k+1} + \vJ_{8,k}\vx_{k} + (\vI - \vJ_{7,k} - \vJ_{8,k})(\vx_{k+1}-\vx_{\ast})
\end{aligned}
\]
With 
\[\vA_{k}: = \vJ_{8,k},\quad \vb_{2,k}:= (\vI - \vJ_{7,k} - \vJ_{8,k})(\vx_{k+1}-\vx_{\ast}),\]one can immediately obtain \eqref{eq:l2o-nesterov-m} the following bounds
\[
\begin{aligned}
 \|\vA_{k}\| &\leq  \sqrt{n}C\\
 \|\vb_{2,k}\| &\leq  \|\vI - \vJ_{7,k} - \vJ_{8,k}\| \|\vx_{k+1}-\vx_{\ast}\| \leq \big( 1+2\sqrt{n}C \big)\|\vx_{k+1}-\vx_{\ast}\| \to 0.
\end{aligned}
\]

\paragraph{Step 3: Proof of (\ref{eq:l2o-nesterov-quadratic}).} To prove \eqref{eq:l2o-nesterov-quadratic}, we assume sequence $\{\vP_{1,k}\}$ is uniformly symmetric positive definite, i.e., the smallest eigenvalues of symmetric positive definite $\{\vP_{1,k}\}$ are uniformly bounded away from zero. Thus, the matrix sequence $\vP_{1,k}^{-1}\vP_{2,k}$ is bounded. Since equation \eqref{eq:l2o-nesterov-d} holds for all bounded matrix sequence $\vB_{k}$, we let 
\[
\vB_{k}:= \vP_{2,k}\vP_{1,k}^{-1},
\]
and obtain
\[
\vx_{k+1} = \vx_{k} - (\vP_{1,k}-\vP_{2,k})\nabla f(\vx_{k}) - \vP_{2,k}\nabla f(\vy_{k}) - \vP_{1,k} \vg_{k+1} + \vP_{2,k}\vP_{1,k}^{-1}(\vy_{k} - \vx_{k}) - \vb_{1,k}.
\]
Therefore, it holds that
\begin{align}
 \vx_{k+1} + \vP_{1,k} \vg_{k+1} &= \vx_{k} - (\vP_{1,k}-\vP_{2,k})\nabla f(\vx_{k}) - \vP_{2,k}\nabla f(\vy_{k}) + \vP_{1,k}^{-1}\vP_{2,k}(\vy_{k} - \vx_{k}) - \vb_{1,k}\\
&= \Big( \vI - \vP_{2,k}\vP_{1,k}^{-1} \Big) \Big(\vx_{k} - \vP_{1,k}\nabla f(\vx_{k}) \Big) + \vP_{2,k}\vP_{1,k}^{-1} \Big(\vy_{k} - \vP_{1,k}\nabla f(\vy_{k})\Big) - \vb_{1,k}\\
&= \Big( \vI - \vP_{2,k}\vP_{1,k}^{-1} \Big) \hat{\vx}_{k} + \vP_{2,k}\vP_{1,k}^{-1} \hat{\vy}_{k} - \vb_{1,k}\\
&= \Big( \vI - \vB_{k} \Big) \hat{\vx}_{k} + \vB_{k} \hat{\vy}_{k} - \vb_{1,k}.
\end{align}
Since $\vg_{k+1} \in \partial r(\vx_{k+1})$, we have $\vx_{k+1}$ yields the following inclusion equation
\[\mathbf{0} \in \partial r(\vx) + \vP_{1,k}^{-1} \left(\vx -  \big( \vI - \vB_{k} \big) \hat{\vx}_{k} - \vB_{k} \hat{\vy}_{k} + \vb_{1,k} \right).\]
Consequently, $\vx_{k+1}$ is the unique minimizer of the following convex optimization problem
\begin{equation}
\label{eq:app-prox-opt}
    \min_{\vx} r(\vx) + \frac{1}{2}\left\| \vx - \big( \vI - \vB_{k} \big) \hat{\vx}_{k} - \vB_{k} \hat{\vy}_{k} + \vb_{1,k} \right\|^2_{\vP_{1,k}^{-1}}.
\end{equation}
Applying the definition of preconditioned proximal operator \eqref{eq:define-prox}, one could immediately get \eqref{eq:l2o-nesterov-quadratic}. The strong convexity of \eqref{eq:app-prox-opt} implies that \eqref{eq:app-prox-opt} admits a unique minimizer, which concludes the uniqueness of $\vx_{k+1}$ and finishes the whole proof.

\section{Other Theoretical Results}
\label{sec:b}
In this section, we study the explicit update rule \eqref{eq:l2o-subgd} in the non-smooth case. To facilitate reading, we rewrite \eqref{eq:l2o-subgd} here:
\begin{equation}
    \label{eq:l2o-subgd-app}
    \vx_{k+1} = \vx_{k} - \vd_{k}(\vx_{k},\vg_{k}),\quad \vg_{k} \in \partial r(\vx_{k}).
\end{equation}
We show that, even for some simple functions, one may not expect to obtain an efficient update rule if $\vd_{k}\in\cD_{C}(\mathbb{R}^n\times\mathbb{R}^n)$. The one-dimensional case is presented in Proposition~\ref{prop:subgra} and the n-dimensional case is presented in Proposition~\ref{prop:subgra2}.

In the one-dimensional case, we consider function $r(x) = |x|$. It has unique minimizer $x_{\ast}=0$. Its subdifferential is:
\begin{equation}
    \label{eq:abs-sub-gra}
    \partial r(x) = \begin{cases}
    \mathrm{sign}(x), & x \neq 0;\\
    [-1,1], & x = 0.
    \end{cases}
\end{equation}
Since $x_{\ast}=0$, the asymptotic fixed point condition is $d_{k}(0,0) \to 0$. Furthermore, we assume all sequences generated by \eqref{eq:l2o-subgd-app} with initial points $x_{0} \in [-1,1]$ converges to $0$ \emph{uniformly}. In another word, there is a uniform convergence rate for all possible sequences. Due to the uniqueness of minimizer, one may expect a good update rule satisfy such uniform convergence.

\begin{proposition}
\label{prop:subgra}
Consider 1-D function $r(x) = |x|$. Suppose we pick $d_{k}$ from $\cD_{C}(\mathbb{R})$ and form a operator sequence $\{d_{k}\}_{k=0}^{\infty}$. If we assume:
 \begin{itemize}[leftmargin=*]
     \item It holds that $d_{k}(0,0) \to 0$ as $k \to \infty$.
     \item Any sequences $\{x_{k}\}$ generated by \eqref{eq:l2o-subgd-app} converges to $0$ \emph{uniformly} for all initial points $x_{0} \in [-1,1]$.
 \end{itemize}
 then there exist $\{p_{k},b_{k}\}_{k=0}^{\infty}$ satisfying
\[d_{k}(x_{k},g_{k} ) = p_{k}g_{k} + b_{k},\quad g_{k} \in \partial r(x_{k}),\quad \text{for all }k=0,1,2,\cdots,\]
$p_{k}\to0$, and $b_{k}\to 0$ as $k\to\infty$.
 \end{proposition}
 
This proposition demonstrates that if $r(x) = |x|$, any update rule in the form of \eqref{eq:l2o-subgd-app} actually equals to subgradient descent method with adaptive step size $p_{k}$ and bias $b_{k}$. The step size $p_{k}$ must be diminishing, otherwise, the uniform convergence would be broken. Diminishing step size usually leads to a slower convergence rate than constant step size. Thus, one may not expect to obtain an efficient update rule in this case.

In the n-dimensional case, we consider a family of $n$-dim function 
\[\cF_{\ell_1}(\mathbb{R}^n) = \left\{\|\vA \vx\|_1: \vA \in \mathbb{R}^{n\times n}\textnormal{, }\|\vA\| \leq 1\textnormal{, and $\vA$ is non-singular}\right\}.\] 
All functions in $\cF_{\ell_1}(\mathbb{R}^n)$ have a unique minimizer $\vx_{\ast}=\mathbf{0}$. Its subdifferential is defined as
\[\partial r(\vx) = \vA^\top \partial \|\vA \vx\|_1.\]
If every element of $\vA \vx$ is non-zero, it holds that
\begin{equation}
    \label{eq:abs2-sub-gra}
    \partial r(\vx) = \vA^\top \mathrm{sign}(\vA \vx).
\end{equation}
As an extension to Proposition~\ref{prop:subgra}, the asymptotic fixed point condition becomes $d_{k}(\mathbf{0},\mathbf{0}) \to \mathbf{0}$. We also assume that all sequences generated by \eqref{eq:l2o-subgd-app} converge to $\mathbf{0}$ \emph{uniformly} for all possible functions $r(\vx)\in \cF_{\ell_1}(\mathbb{R}^n)$ due to these function share the same minimizer.

 \begin{proposition}
\label{prop:subgra2}
Suppose we pick $\vd_{k}$ from $\cD_{C}(\mathbb{R}^n\times\mathbb{R}^n)$ and form an operator sequence $\{\vd_{k}\}_{k=0}^{\infty}$. If  we assume: 
 \begin{itemize}[leftmargin=*]
     \item It holds that 
     $d_{k}(\mathbf{0},\mathbf{0}) \to \mathbf{0}$.
     \item Any sequences $\{\vx_{k}\}$ generated by \eqref{eq:l2o-subgd-app} converges to $\mathbf{0}$ \emph{uniformly} for all $r(\vx)\in \cF_{\ell_1}(\mathbb{R}^n)$ and all initial points $\vx_{0} \in [-1,1]^n$. 
 \end{itemize}
 then there exist $\{\vP_{k},\vb_{k}\}_{k=0}^{\infty}$ satisfying
\[\vd_{k}(\vx_{k},\vg_{k} ) = \vP_{k}\vg_{k} + \vb_{k},\quad \vg_{k} \in \partial r(\vx_{k}),\quad \text{for all }k=0,1,2,\cdots,\]
where $\vP_{k}\to \mathbf{0}$ and $\vb_{k}\to \mathbf{0}$ as $k\to\infty$.
 \end{proposition}
 The conclusion is similar to Proposition~\ref{prop:subgra}. The preconditioner $\vP_{k}$ goes smaller and smaller as $k \to \infty$, which means the update step size should be diminishing. The convergence rate gets slower and slower as $k$ increases. Thus, the explicit update rule \eqref{eq:l2o-subgd-app} is not efficient.

 \subsection{Proof of Proposition~\ref{prop:subgra}}
 
  \begin{proof}
 Following the same proof line with that of Theorem~\ref{theo:gd}, we can get the conclusion: for any sequence $\{d_{k}\}_{k=0}^{\infty}$ satisfying the conditions described in Proposition~\ref{prop:subgra}, there exists a sequence $\{p_{k},b_{k}\}_{k=0}^{\infty}$ such that
 \[
 x_{k+1} = x_{k} - p_{k}g_{k} - b_{k},
 \]
 where $g_{k} \in \partial r(x_{k})$ and $|p_{k}| \leq C$ and $b_{k}\to 0$. 
 
 Then we want to show that, as long as all sequences $\{x_{k}\}$ generated by \eqref{eq:l2o-subgd} \emph{uniformly} converges to $x_\ast$, it must hold that $p_{k}\to 0$. We show this by contradiction and assume $p_{k}$ does not converge to zero. In another word, there exist a fixed real number $\varepsilon>0$ and a sub-sequence of $\{p_{k}\}$ such that 
 \[|p_{k_l}| > \varepsilon,\quad l = 1,2,\cdots. \]
 Now we claim that: given $\{p_{k},b_{k}\}_{k=0}^{\infty}$, for any $\hat{k} > 0$, there exits an initial point $x_{0}$ such that $x_{k} \neq 0$ for all $k \leq \hat{k}$. The proof is as follows:
 \begin{itemize}[leftmargin=*]
     \item Given $x_{0}\neq 0$, we have $g_{0} = 1$ or $g_{0} = -1$ due to \eqref{eq:abs-sub-gra}.
     \item To guarantee $x_{1}\neq 0$, it's enough that $x_{0} + p_{0} - b_{0}\neq 0, x_{0} + p_{0} - b_{0}\neq 0$.
     \item Define \[\cX_{1} := \{b_{0}+p_{0},b_{0}-p_{0}\}.\] As long as $x_{0} \not \in \cX_{0} \bigcup \cX_{1}$, we can guarantee $x_{0}\neq 0$ and $x_{1} \neq 0$.
     \item Define \[\cX_{2} := \{b_{0}+p_{0}+b_{1}+p_{1},b_{0}+p_{0}+b_{1}-p_{1},b_{0}-p_{0}+b_{1}+p_{1},b_{0}-p_{0}+b_{1}-p_{1}\}.\] As long as $x_{0} \not \in \cX_{0} \bigcup \cX_{1} \bigcup \cX_{2}$, we can guarantee $x_{0}\neq 0$ and $x_{1} \neq 0$ and $x_{2} \neq 0$.
     \item $\cdots$
 \end{itemize}
 Repeat the above statement for $\hat{k}$ times, we obtain: $x_{0} \not \in \bigcup\limits_{k\leq \hat{k}}\cX_{k}$ implies $x_{k} \neq 0$ for all $k \leq \hat{k}$, where the set $\cX_{k}$ contains $2^k$ elements. Thus, $x_{0}$ can be chosen almost freely within $[-1,1]$ excluding a set with a finite number of elements. The claim is proven.
 
 With $\hat{k} = k_l$, we conclude that, for all $l=1,2,\cdots$, there exists an initial point $x_{0}$ such that $x_{k} \neq 0$ for all $k \leq k_l$. Consequently, it holds that $g_{k_l} = 1$ or $g_{k_l} = -1$. Then, 
 \[|x_{k_l+1} - x_{k_l}| \geq |p_{k_l}g_{k_l}| - |b_{k_l}| = \varepsilon - |b_{k_l}|,\]
 which contradicts with the fact that $b_{k}\to 0$ and $x_{k}\to 0$ uniformly for all initial points. This completes the proof for $p_{k}\to 0$.
 \end{proof}
 
 \subsection{Proof of Proposition~\ref{prop:subgra2}}
 \begin{proof}
 The proof of Proposition~\ref{prop:subgra2} extends the proof of Proposition~\ref{prop:subgra} and follows a similar proof sketch. But the $n$-dim case is much more complicated than the $1$-dim case. Consequently, we need stronger assumptions: $\vx_{k}$ converges uniformly not only for all initial points, but also for a family of objective function $f \in \cF_{\ell_1}(\mathbb{R}^n)$. 
 
 In our proof, we denote $(\vA\vx)^{i}$ as the $i$-th element of vector $\vA\vx$ and the index of a matrix is denoted by $(:,:)$. For example, $\vA(i,:)$ means the $i$-th row of $\vA$; $\vA(:,i)$ means the $i$-th column of $\vA$.
 
 Following the same proof line with that of Theorem~\ref{theo:gd}, we can get the conclusion (similar with Proposition~\ref{prop:subgra}): for any sequence $\{\vd_{k}\}_{k=0}^{\infty}$ satisfying the conditions described in Proposition~\ref{prop:subgra2}, there exists a sequence $\{\vP_{k},\vb_{k}\}_{k=0}^{\infty}$ such that
 \[\vx_{k+1} = \vx_{k} - \vP_{k}\vg_{k} - \vb_{k},\]
 where $\vg_{k} \in \partial r(\vx_{k})$ and $\|\vP_{k}\| \leq \sqrt{n} C$ and $\vb_{k}\to \mathbf{0}$. It's enough to show that $\vP_{k}\to\mathbf{0}$.
 
 Before proving $\vP_{k}\to\mathbf{0}$, we first claim and prove an statement: given $\{\vP_{k},\vb_{k}\}_{k=0}^{\infty}$ and any $r(\vx) = \|\vA \vx\|_1 \in \cF_{\ell_1}(\mathbb{R}^n)$ and any $\hat{k} > 0$, there exits an initial point $\vx_{0}$ such that $(\vA \vx_{k})^i \neq 0$ for all $k \leq \hat{k}$ and $i=1,2,\cdots,n$. The proof is as follows:
 \begin{itemize}[leftmargin=*]
     \item To guarantee $(\vA \vx_{0})^i \neq 0$, $\vx_{0}$ must satisfy:
     \[\vx_{0} \not \in \cX_{0}^i = \{\vx: \vA(i,:)\vx = 0\}.\]
     \item Given $(\vA \vx_{0})^i \neq 0, 1\leq i \leq n$, we have $\vg_{0} = \vA^\top  \mathrm{sign}(\vA \vx_{0})$, where $ \mathrm{sign}(\vA \vx_{0}) \in \{1,-1\}^n$, due to \eqref{eq:abs2-sub-gra}.
     \item To guarantee $(\vA \vx_{1})^i \neq 0$, it's enough that $\vA(i,:)(\vx_{0} - \vP_{0}\vA^\top \vs - \vb_{0}) \neq 0$ for all $\vs \in \{1,-1\}^n$.
     \item Define 
     \[\cX_{1}^i = \{\vx: \vA(i,:)(\vx - \vP_{0}\vA^\top \vs_{0} - \vb_{0}) = 0 \textnormal{ for some }\vs_{0} \in \{1,-1\}^n\}.\]
     As long as $\vx_{0} \not \in  \bigcup\limits_{1\leq i \leq n, 0 \leq k\leq 1} \cX_{k}^i $, we guarantee that $(\vA \vx_{k})^i \neq 0$ for all $k \leq 1$ and $i=1,2,\cdots,n$.
     \item Define 
     \[\cX_{2}^i = \{\vx: \vA(i,:)(\vx - \vP_{0}\vA^\top \vs_{0} - \vb_{0} - \vP_{1}\vA^\top \vs_{1} - \vb_{1}) = 0 \textnormal{ for some }\vs_{0},\vs_{1} \in \{1,-1\}^n\}.\]
     As long as $\vx_{0} \not \in  \bigcup\limits_{1\leq i \leq n, 0 \leq k\leq 2} \cX_{k}^i $, we guarantee that $(\vA \vx_{k})^i \neq 0$ for all $k \leq 2$ and $i=1,2,\cdots,n$.
    \item $\cdots$
 \end{itemize}
Repeat the above statement for $\hat{k}$ times, we obtain: $\vx_{0} \not \in \bigcup\limits_{1\leq i \leq n, 0 \leq k\leq \hat{k}} \cX_{k}^i$ implies the conclusion we want. Moreover, the set $\cX_{k}^i$ has measurement zero in the space $\mathbb{R}^n$ due to the fact that each row of matrix $\vA$: $\vA(i,:)$ is not zero ($\vA$ is non-singular). Finite union of $\cX_{k}^i$ also has measurement zero. Thus, $\vx_{0}$ can be chosen almost freely in $[-1,1]^n$ excluding a set with zero measurements. The claim is proven.
 
 Now we show $\vP_{k}\to\mathbf{0}$ by contradiction. Assume there exist a fixed real number $\varepsilon>0$ and a sub-sequence of $\{\vP_{k}\}$ such that 
 \begin{equation}
     \label{eq:p_lower_bound}
     \|\vP_{k_l}\| > \varepsilon,\quad l = 1,2,\cdots.
 \end{equation}
  Conduct SVD on $\vP_{k_l}$: \[\vP_{k_l} = \vU_{k_l} \Sigma_{k_l} \vV_{k_l}^\top  = \vU_{k_l} \begin{bmatrix} \sigma^1_{k_l} & & & \\
  & \sigma^2_{k_l} & & \\ & & \ddots & \\ & & & \sigma^n_{k_l}\end{bmatrix} \vV_{k_l}^\top . \]
  Inequality \eqref{eq:p_lower_bound} implies the largest singular value of $\vP_{k_l}$ should be greater than $\varepsilon$. WLOG, we assume $\sigma^1_{k_l}$ is the largest one and, consequently, $\sigma^1_{k_l} > \varepsilon$. Given such $\vV_{k_l}$, we define
  \[f_{k_l}(\vx) = \|\vA_{k_l} \vx\|_1, \quad \vA_{k_l} = \vV_{k_l}^\top . \]
  It's easy to check that $f_{k_l}\in \cF_{\ell_1}(\mathbb{R}^n)$. 
  
  Given $\{\vP_{k_l}, \vb_{k_l}\}_{l=0}^{\infty}$ and function $f_{k_l}(\vx)$, we take an initial point $\vx_{0} \not \in \bigcup\limits_{1\leq i \leq n, 0 \leq k\leq k_l} \cX_{k}^i$, then we have $(\vA_{k_l} \vx_{k_l})^i \neq 0$. Thus, it holds  that
  \[\vx_{k_l+1} = \vx_{k_l} - \vP_{k_l}\vA_{k_l}^\top \mathrm{sign}(\vA_{k_l} \vx_{k_l}) - \vb_{k_l} = \vx_{k_l} - \vP_{k_l}\vA_{k_l}^\top \vs_{k_l} - \vb_{k_l}\]for some $\vs_{k_l}\in\{1,-1\}^n$. Rewrite the second term on the right-hand side
  \[\vP_{k_l}\vA_{k_l}^\top \vs_{k_l} = \vU_{k_l} \begin{bmatrix} \sigma^1_{k_l} & & & \\
  & \sigma^2_{k_l} & & \\ & &  \ddots & \\ & & & \sigma^n_{k_l} \end{bmatrix} \vs_{k_l} = \vU_{k_l} \begin{bmatrix}  \sigma^1_{k_l}s_{k_l}^{1} \\ \sigma^2_{k_l}s_{k_l}^{2} \\ \vdots \\ \sigma^n_{k_l}s_{k_l}^{n}  \end{bmatrix}.\]
  Its norm is lower bounded by
  \[\big\| \vP_{k_l}\vA_{k_l}^\top \vs_{k_l} \big\| \geq |\sigma^1_{k_l}s_{k_l}^{1}| > \varepsilon.\]
  Then we get
  \[\|\vx_{k_l+1} - \vx_{k_l}\| \geq \big\| \vP_{k_l}\vA_{k_l}^\top \vs_{k_l} \big\| - \|\vb_{k_l}\|,\]
  which contradicts with the fact that $\vb_{k}\to \mathbf{0}$ and $\vx_{k}\to\mathbf{0}$ uniformly. $\vP_{k}\to\mathbf{0}$ is proved and this finishes the proof.
 \end{proof}

 \section{Scheme~(\ref{eq:final-scheme}) Covers Many Schemes in the Literature}
In this section, we show that our proposed scheme \eqref{eq:final-scheme} covers FISTA~\citep{beck2009fast}, PGD with variable metric~\citep{park2020variable}, Step-LISTA~\citep{ablin2019learning}, and Ada-LISTA~\citep{aberdam2021ada}. We will show them one by one. To facilitate reading, we rewrite \eqref{eq:final-scheme} here:
\[
     \begin{aligned}
    \hat{\vx}_{k} &= \vx_{k} - \vp_{k}\odot\nabla f(\vx_{k}),\\
    \hat{\vy}_{k} &= \vy_{k} - \vp_{k}\odot\nabla f(\vy_{k}),\\
    \vx_{k+1} &= \prox_{r, \vp_{k}} \Big( (\mathbf{1} - \vb_{k}) \odot \hat{\vx}_{k} + \vb_{k} \odot \hat{\vy}_{k} - \vb_{1,k} \Big),\\
     \vy_{k+1} &= \vx_{k+1} + \va_{k} \odot (\vx_{k+1} - \vx_{k}) + \vb_{2,k}.
    \end{aligned}
\]

\paragraph{FISTA.} The update rule of FISTA (with constant step size) writes
\begin{equation}
    \label{eq:fista}
    \begin{aligned}
     \vy_{k+1} &= \prox_{r, (1/L)\mathbf{1}} \Big( \vx_{k} - \frac{1}{L} \nabla f(\vx_{k})\Big),\\
    t_{k+1} &= \frac{1 + \sqrt{1+4t_{k}^2}}{2},\\
     \vx_{k+1} &= \vy_{k+1} + \frac{t_{k}-1}{t_{k+1}} (\vy_{k+1} - \vy_{k}),
    \end{aligned}
\end{equation}
where $L$ is the Lypuschitz constant of $\nabla f$. 
Thus, as long as $\vb_{k}=\mathbf{1}$, $\vb_{1,k}=\vb_{2,k}=\mathbf{0}$, $\vp_{k} = (1/L)\mathbf{1}$, and $\va_{k} = \frac{t_{k}-1}{t_{k+1}}\mathbf{1}$, \eqref{eq:final-scheme} is equal to \eqref{eq:fista}.

\paragraph{PGD.} PGD with variable metric writes
\begin{equation}
    \label{eq:pgd-pre}
    \vx_{k+1} = \prox_{r, \vp_{k}} \Big( \vx_{k} - \vp_{k} \odot \nabla f(\vx_{k})\Big).
\end{equation}
If $\vb_{k}=\va_{k}=\vb_{1,k}=\vb_{2,k}=\mathbf{0}$, \eqref{eq:final-scheme} reduces to \eqref{eq:pgd-pre}. 

\paragraph{Step-LISTA.} The update rule of Step-LISTA writes
\begin{equation}
    \label{eq:step-lista}
    \vx_{k+1} = \sigma\Big(  \vx_{k} - p_{k} \vA^\top\big( \vA \vx - \vb \big), \theta_{k}  \Big).
\end{equation}
If the objective function is taken as standard LASSO and we take $\vp_{k} = p_{k}\mathbf{1}$
\[F(\vx) = \underbrace{\frac{1}{2}\|\vA\vx-\vb\|^2}_{f(\vx)} + \underbrace{\lambda \|\vx\|_1}_{r(\vx)},\]
then we have
\[\nabla f(\vx) = \vA^\top (\vA\vx - \vb),\quad \prox_{r,\vp_{k}}(\vx) = \sigma(\vx, \lambda p_{k}).\]We want to show that, for any sequence $\{\theta_{k}\}_{k=0}^{\infty}$, there exists a sequence $\{\va_{k}, \vb_{k},\vb_{1,k},\vb_{2,k}\}_{k=0}^{\infty}$ such that
\eqref{eq:final-scheme} equals to \eqref{eq:step-lista}.
\begin{proof}
Take $\vp_{k} = p_{k}\mathbf{1}$ and $\va_{k}=\vb_{k}=\vb_{2,k}=\mathbf{0}$, \eqref{eq:final-scheme} reduces to
\[\vx_{k+1} = \sigma\Big(\vx_{k} - p_{k}\vA^\top (\vA\vx_{k} - \vb) - \vb_{1,k}, \lambda p_{k}\Big).\]
Define 
\[\hat{\vx}_{k} = \vx_{k} - p_{k}\vA^\top (\vA\vx_{k} - \vb).\]

If $\theta_{k} > \lambda p_{k}$, define $\vb_{1,k}$ component-wisely as (Here $(\vb_{1,k})_i$ means the $i$-th component of vector $\vb_{1,k}$):
\[(\vb_{1,k})_i = \sign\big((\hat{\vx}_{k})_i\big) \min(\theta_{k} - \lambda p_{k}, \big|(\hat{\vx}_{k})_i\big|).\]
Then one can check that 
\begin{equation}
\label{eq:app-equal-theta}
    \sigma\big(  \hat{\vx} - \vb_{1,k}, \lambda p_{k} \big) = \sigma\big(  \hat{\vx}, \theta_{k} \big),
\end{equation}
where the role of $\vb_{1,k}$ is enhancing the soft thereshold from $\lambda p_{k}$ to $\theta_{k}$.

If $\theta_{k} < \lambda p_{k}$, define $\vb_{1,k}$ with: 
\[(\vb_{1,k})_i = \begin{cases}\sign((\hat{\vx}_{k})_i) (\theta_{k} - \lambda p_{k}), & (\hat{\vx}_{k})_i > \theta_{k}\\
0, & (\hat{\vx}_{k})_i \leq \theta_{k}\end{cases},\]
then \eqref{eq:app-equal-theta} also holds in this case. 

With the $\{p_{k}, \va_{k}, \vb_{k},\vb_{1,k},\vb_{2,k}\}$ defined above, it holds that \eqref{eq:final-scheme} equals to \eqref{eq:step-lista}, which finishes the proof.
\end{proof}

\paragraph{Ada-LISTA.} The update rule of Ada-LISTA with single weight matrix writes
\begin{equation}
    \label{eq:ada-lista}
    \vx_{k+1} = \sigma\Big(  \vx_{k} - p_{k} \vA^\top \vM^\top \vM \big( \vA \vx - \vb \big), \theta_{k}  \Big).
\end{equation}
If the objective function is taken as LASSO with a learned dictionary $\vM$:
\[F(\vx) = \underbrace{\frac{1}{2}\Big\|\vM\big(\vA\vx-\vb\big)\Big\|^2}_{f(\vx)} + \underbrace{\lambda \|\vx\|_1}_{r(\vx)},\]
and we follow the same proof line as that of Step-LISTA, then we obtain that \eqref{eq:final-scheme} covers \eqref{eq:ada-lista}.

\section{Examples of Explicit Proximal Operator}
\label{sec:example-prox}

As long as one can evaluate $\nabla f$ and the proximal operator $\prox_{r, \vp_{k}}$, the update rule \eqref{eq:final-scheme} is applicable. The gradient $\nabla f$ is accessible since $f \in \cF_{L}(\mathbb{R}^n)$. For a broad class of $r(\vx)$, the operator $\prox_{r,\vp_{k}}$ has efficient explicit formula. We list some examples here and more examples can be found in \citep{parikh2014proximal,park2020variable}.
 \begin{itemize}[leftmargin=*]
     \item ($\ell_1$-norm) Suppose $r(\vx) = \lambda \|\vx\|_1$, then the proximal operator is a scaled soft-thresholding operator that is component-wisely defined as $\big(\prox_{r, \vp_{k}}(\vx)\big)_i := \mathrm{sign}(x_i)\max(0, |x_i| - \lambda (\vp_{k})_i),$ for $1 \leq i \leq n$.
     \item (Non-negative constraint)
     Suppose $r(\vx) = \iota_{\cX}(\vx)$ where $\cX = \{\vx \in \mathbb{R}^n: x_i \geq 0, 1\leq i \leq n\}$ and $\iota_{\cX}$ is the indicator function (i.e., $\iota_{\cX}(\vx) = 0$ if $\vx\in\cX$; $\iota_{\cX}(\vx) = +\infty$ otherwise), then $\prox_{r,\vp_{k}}$ is 
     component-wisely defined as $\big(\prox_{r, \vp_{k}}(\vx)\big)_i := \max(0, x_i),~~ 1\leq i \leq n$.
     \item (Simplex constraint) Suppose $r(\vx) = \iota_{\cX}(\vx)$ where $\cX = \{\vx \in \mathbb{R}^n: x_i \geq 0, 1\leq i \leq n; \bm{1}^\top\vx = 1.\}$, then $\prox_{r,\vP_{k}}$ is component-wisely defined as $\big(\prox_{r, \vp_{k}}(\vx)\big)_i:= \max(0, x_i - \xi (\vp_{k})_i),~ 1\leq i \leq n$, where $\xi\in\mathbb{R}^+$ can be determined efficiently via bisection.
 \end{itemize}

\section{Details in Our Experiments}
\label{sec:details-exp}

\paragraph{LASSO Regression.} In this paragraph, we provide the details of the LASSO benchmarks used in this paper.
\begin{itemize}
    \item (Synthetic data). Each element in $\vA \in \mathbb{R}^{250\times500}$ is sampled \emph{i.i.d.}~from the normal distribution, and each column of $\vA$ is normalized to have a unit $\ell_2$-norm. Then we randomly generate sparse vector $\vx_{\ast} \in \mathbb{R}^{500}$. In each sparse vector, we first uniformly sample $50$ out of $500$ entries to be nonzero, and the value of each nonzero is sampled independently from the normal distribution. With $\vA$ and $\vx_{\ast}$, we generate $\vb$ with $\vb=\vA\vx_{\ast}$. Such tuple $(\vA,\vb,\lambda)$ forms an instance of LASSO, and we repeatedly generate $(\vA,\vb,\lambda)$ in the above approach. The training set includes $32,000$ independent optimization problems and the testing set includes $1,024$ independent optimization problems. We take $\lambda=0.1$ for all synthetic LASSO instances. 
    \item (Real data). We extract $1000$ patches with size $8\times8$ at random positions from testing images that are randomly chosen from BSDS500~\citep{MartinFTM01}. Each patch is flattened to a vector in space $\mathbb{R}^{64}$ and normalized and mean-removed. Then we conduct K-SVD \citep{aharon2006k} to obtain a dictionary $\vA\in\mathbb{R}^{64\times128}$. Each vector in $\mathbb{R}^{64}$ can be viewed as an instance of $\vb$ in LASSO \eqref{eq:lasso}. With a shared matrix $\vA$, we construct $1000$ instances of LASSO and test our methods on them. We take $\lambda=0.5$ for all real-data LASSO instances.
\end{itemize}

\paragraph{Logistic Regression.} Given a set of training examples $\{(\va_i, b_i) \in \mathbb{R}^n\times\{0,1\}\}_{i=1}^m$, the objective function of the $\ell_1$-regularized logistic regression problem is defined as
\begin{equation}
\label{eq:logreg}
    \min_{\vx\in\mathbb{R}^n} F(\vx) = -\frac{1}{m}\sum_{i=1}^m [b_i \log(h(\va_i^\top\vx)) + (1-b_i)\log(1-h(\va_i^\top\vx))] + \lambda \|\vx\|_1,
\end{equation}
where $h(c)=1/(1+e^{-c})$ is the logistic function. 
For each logistic regression problem, we generate 1,000 feature vectors, each of which $\va \in \mathbb{R}^{50}$ is sampled \emph{i.i.d.}~from the normal distribution. Then we randomly generate a sparse vector $\vx_{\ast} \in \mathbb{R}^{50}$ and uniformly sample $20$ out of its $50$ entries to be nonzero, and the value of each nonzero is sampled independently from the normal distribution. With $\va_i$ and $\vx_{\ast}$, we generate the binary classification label $b_i$ with $b_i=\mathbbm{1}(\va_i^\top\vx_{\ast}>=0)$, where $\mathbbm{1}(\cdot)$ is the indicator function. Finally, we fix $\lambda=0.1$. Such a pair $(\{(\va_i, b_i)\}_{i=1}^m,\lambda)$ forms an instance of logistic regression with $\ell_1$ regularization, and we repeatedly generate such pairs in the above approach. The training set includes $32,000$ independent optimization problems and the testing set includes $1,024$ independent optimization problems.

We evaluate L2O optimizers (trained on synthesized datasets) on two real-world datasets from the UCI Machine Learning Repository \citep{Dua2019}: (i) \emph{Ionosphere} containing 351 samples of 34 features, and (ii) \emph{Spambase} containing 4,061 samples of 57 features. The results on the Ionosphere dataset is shown in the main text Figure~\ref{fig:logistic-real}. And here we present the results on the Spambase dataset in Figure~\ref{fig:logistic-real-spambase}. Our observation is consistent: L2O-PA is superior in stability and fast convergence compared to all other baselines and is almost 20$\times$ faster than FISTA.

\begin{figure}[h]
    \centering
    \includegraphics[width=0.46\linewidth]{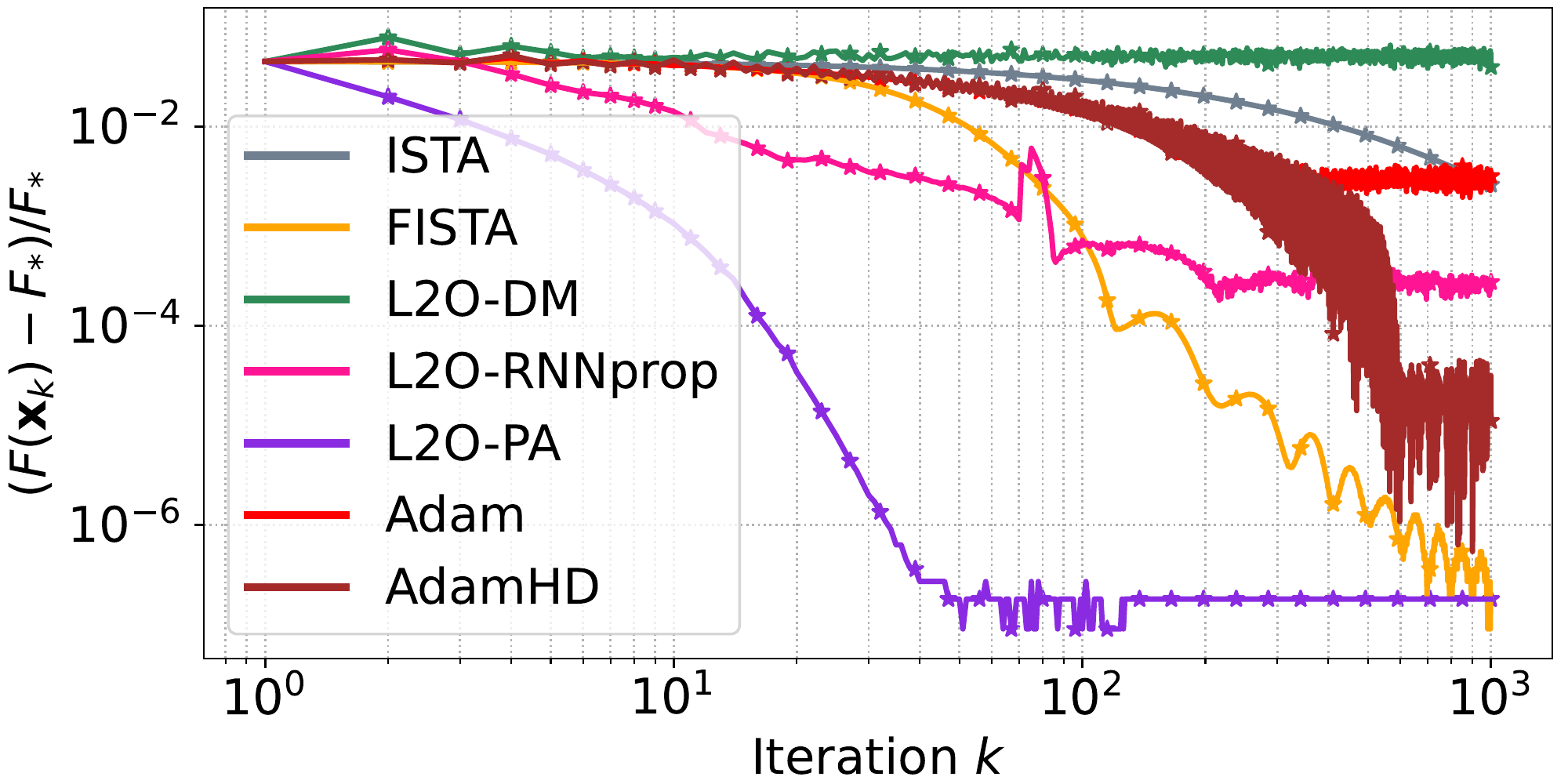}
    \vspace{-1mm}
    \caption{Logistic: Train on synthetic data and test on real data (Spambase).}
    \label{fig:logistic-real-spambase}
\end{figure}

\section{Extra Experiments}
\label{sec:extra-exp}

\paragraph{Running Time Comparison.} Considering that HPO methods such as AdamHD do not require LSTM and consume less time per iteration compared to L2O-PA, we compared the running time of our proposed method L2O-PA and AdamHD in Table~\ref{tab:runtime}. The experiment settings follow those in Section 4.2. In these tables, "Time/Iters" represents the average time consumed for each iteration across the $1024$ testing examples. The "Iters" column indicates the number of iterations needed to achieve the specified precision, while the "Time" column denotes the time required to reach that precision. ``N/A" is used when AdamHD cannot attain a precision of $10^{-6}$ and ``Gap" means the optimality gap $(F(\vx_{k})-F_\ast)/F_\ast$. Table~\ref{tab:runtime} clearly shows that L2O-PA requires much less time than AdamHD, even though its per-iteration complexity is higher than that of AdamHD.

\begin{table}[h]
\vspace{-0.2em}
\centering
\caption{Runtime Comparison between L2O-PA and AdamHD.}
\label{tab:runtime}
\vspace{0.8em}
\begin{tabular}{cccccc}
\toprule
\multicolumn{1}{c|}{}       & \multicolumn{1}{c|}{}           & \multicolumn{2}{c|}{Stopping condition: Gap \textless $10^{-3}$} & \multicolumn{2}{c}{Stopping condition: Gap \textless $10^{-6}$} \\ 
\midrule
\multicolumn{1}{c|}{}       & \multicolumn{1}{c|}{Time/Iters} & \multicolumn{1}{c|}{Iters}    & \multicolumn{1}{c|}{Total Time}   & \multicolumn{1}{c|}{Iters}              & Total Time              \\ \bottomrule
\multicolumn{6}{c}{\bf LASSO (Synthetic)}                                                      \\ 
\toprule
\multicolumn{1}{c|}{L2O-PA} & \multicolumn{1}{c|}{$2.31 \times 10^{-2}$  ms}           & \multicolumn{1}{c|}{$21$}         & \multicolumn{1}{c|}{$0.485$ ms}       & \multicolumn{1}{c|}{$42$}                   &        $0.971$ ms           \\ 
\midrule
\multicolumn{1}{c|}{AdamHD} & \multicolumn{1}{c|}{$8.09 \times 10^{-3}$  ms}           & \multicolumn{1}{c|}{$477$}         & \multicolumn{1}{c|}{$3.858$ ms}       & \multicolumn{1}{c|}{N/A}                   &        N/A           \\ 
\bottomrule
\multicolumn{6}{c}{\bf Logistic (Spambase)}                                                \\ 
\toprule
\multicolumn{1}{c|}{L2O-PA} & \multicolumn{1}{c|}{$7.845 \times 10^{-1}$  ms}           & \multicolumn{1}{c|}{$10$}         & \multicolumn{1}{c|}{$7.845$ ms}       & \multicolumn{1}{c|}{$33$}                   &        $25.89$ ms           \\ 
\midrule
\multicolumn{1}{c|}{AdamHD} & \multicolumn{1}{c|}{$2.605 \times 10^{-1}$  ms}           & \multicolumn{1}{c|}{$390$}         & \multicolumn{1}{c|}{$101.6$ ms}       & \multicolumn{1}{c|}{N/A}                   &       N/A            \\ 
\bottomrule
\end{tabular}
\vspace{-1mm}
\end{table}

\paragraph{Large-Scale LASSO.} To evaluate our model's performance on large problems, we generate $256$ independent LASSO instances of size $2500\times5000$, following the same distribution described in Section~\ref{sec:details-exp}. All the learning models are trained with instances of size $250\times500$ and tested on these $256$ large testing problems. The results, reported in Figure~\ref{fig:lasso-ultra-large}, clearly demonstrate that our proposed (L2O-PA) exhibits a superior ability to generalize to large problems.

\begin{figure}[h]
    \centering
    \includegraphics[width=0.46\linewidth]{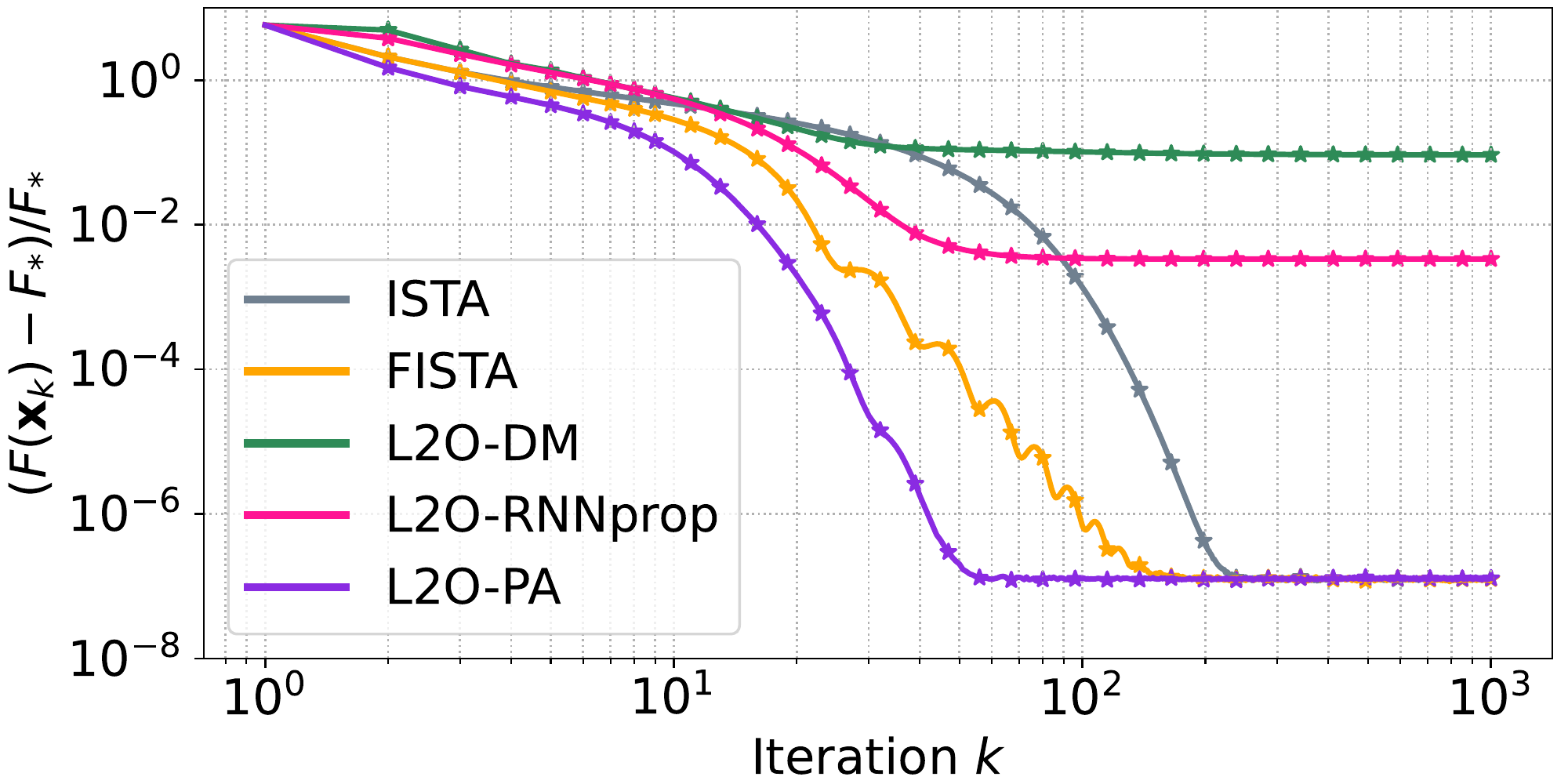}
    \vspace{-1mm}
    \caption{LASSO: Train with small instances and test on large instances.}
    \label{fig:lasso-ultra-large}
\end{figure}

\paragraph{Longer-Horizen Experiments.} To test the performance of our method with longer horizons, say $10^{4}$ iterations, we applied our proposed approach (L2O-PA) to a logistic regression problem with $\ell_1$ regularization on CIFAR-10 for classification, and compared with other baselines. We randomly sampled $5000$ images in the training set of CIFAR-10 (500 images from each class out of 10) for the logistic regression, which followed a similar manner as in \citep{cowen2019lsalsa}. Each image was normalized and flattened into a $3072$-dim vector (i.e., $3\times32\times32$). Since the feature dimension is significantly higher than what we considered in Section 4.2, we used a much smaller regularization coefficient $\lambda=10^{-4}$ to avoid all zero solutions.

We trained learning-based models (L2O-PA, L2O-DM and L2O-RNNprop) on a set of synthesized $\ell_1$-regularized logistic regression tasks for binary classification with $\lambda=10^{-4}$ in the same way as we did in the second part of Section 4.2 and described in Section~\ref{sec:details-exp}. Each logistic regression task contains a dataset of $1000$ samples with $50$ features. After training, all models are directly applied to optimizing the 10-class logistic regression on CIFAR-10 for $10^{4}$ steps. The results, with comparisons to ISTA and FISTA, are shown in Figure~\ref{fig:logistic-real-cifar10}. 
From the results we can see that:
\begin{itemize}
    \item Our method, L2O-PA, converged quite stably in both near and further horizons compared to L2O-DM and L2O-RNNprop, which fluctuated wildly in later iterations. This shows the impressive generalization ability of L2O-PA considering the fact that it was trained in short-horizon settings (100 optimization steps).
    \item Compared to FISTA, L2O-PA can still achieve impressive acceleration in earlier iterations (25 iterations of L2O-PA comparable to FISTA at $1000+$ iterations, and $300$ steps vs $2500+$ steps for FISTA to reach $10^{-2}$ relative error).
\end{itemize}
Therefore, the conclusion is that L2O-PA can still generalize well, to some extent, to longer-horizon tasks even if it is trained in short-horizon settings but it does struggle to converge fast in later iterations. We are happy to include this discussion in the main text as limitations and improve in this direction in the future.
\begin{figure}[h]
    \centering
    \includegraphics[width=0.46\linewidth]{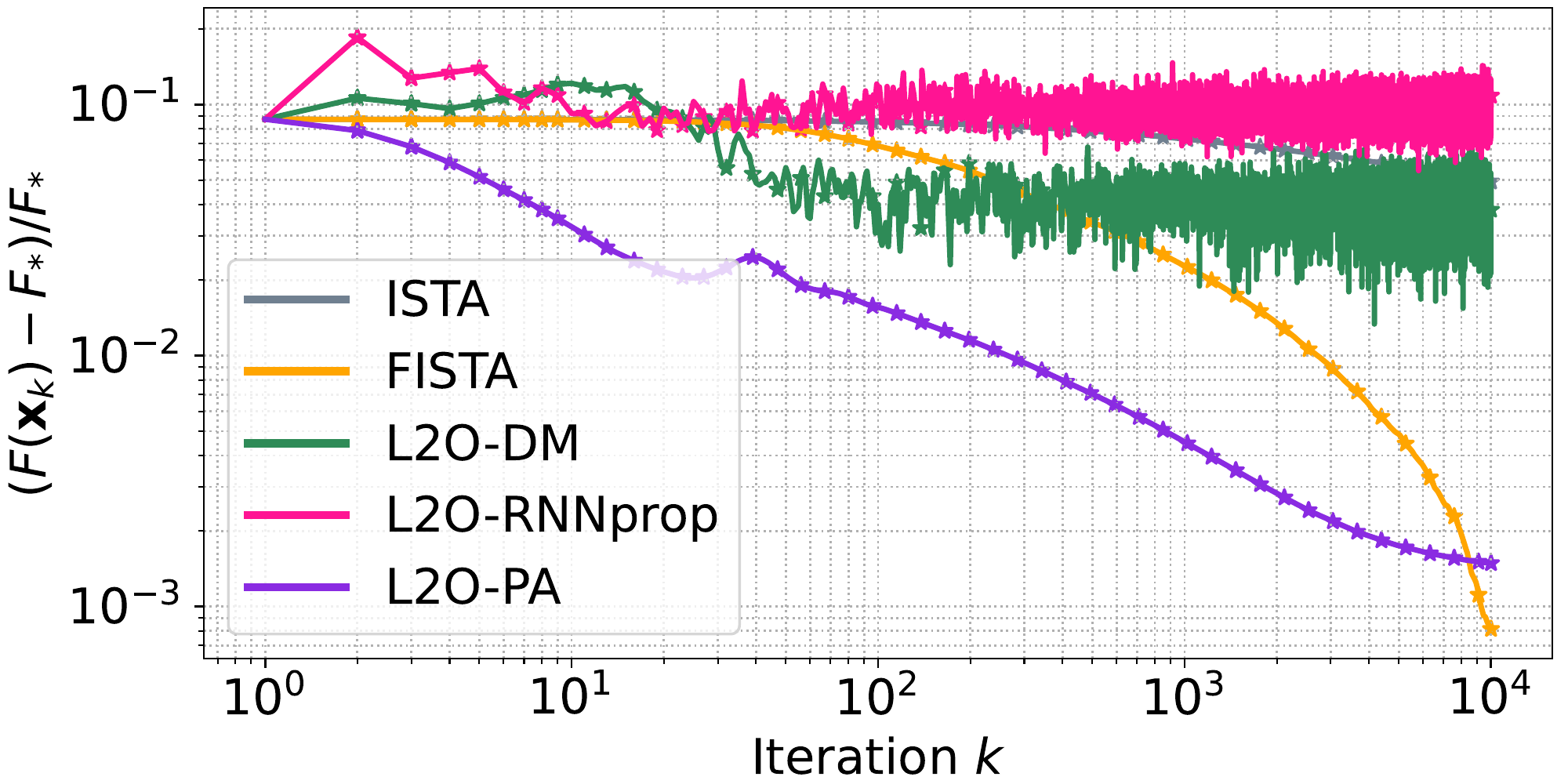}
    \vspace{-1mm}
    \caption{Logistic: Train on synthetic data and test on real data (CIFAR-10).}
    \label{fig:logistic-real-cifar10}
\end{figure}

\paragraph{More-Challenging OOD Experiment.} To further test the generalization performance of our method, we conduct an even more challenging OOD experiment. We directly tested learned optimizers, which were trained on synthetic LASSO problems, on synthetic $\ell_1$-regularized Logistic Regression. This setting renders changes in the objective function and thus the structure of the loss function. The results are shown in Figure~\ref{fig:logistic-full-ood}. We can see that \textbf{L2O-PA-LASSO}, the model that was trained with LASSO problems, is able to converge stably at a faster speed than that of  FISTA and other L2O competitors (except for L2O-PA) on Logistic Regression. It is worth noting that all other L2O competitors are trained directly on Logistic Regression.
\begin{figure}[h]
    \centering
    \includegraphics[width=0.46\linewidth]{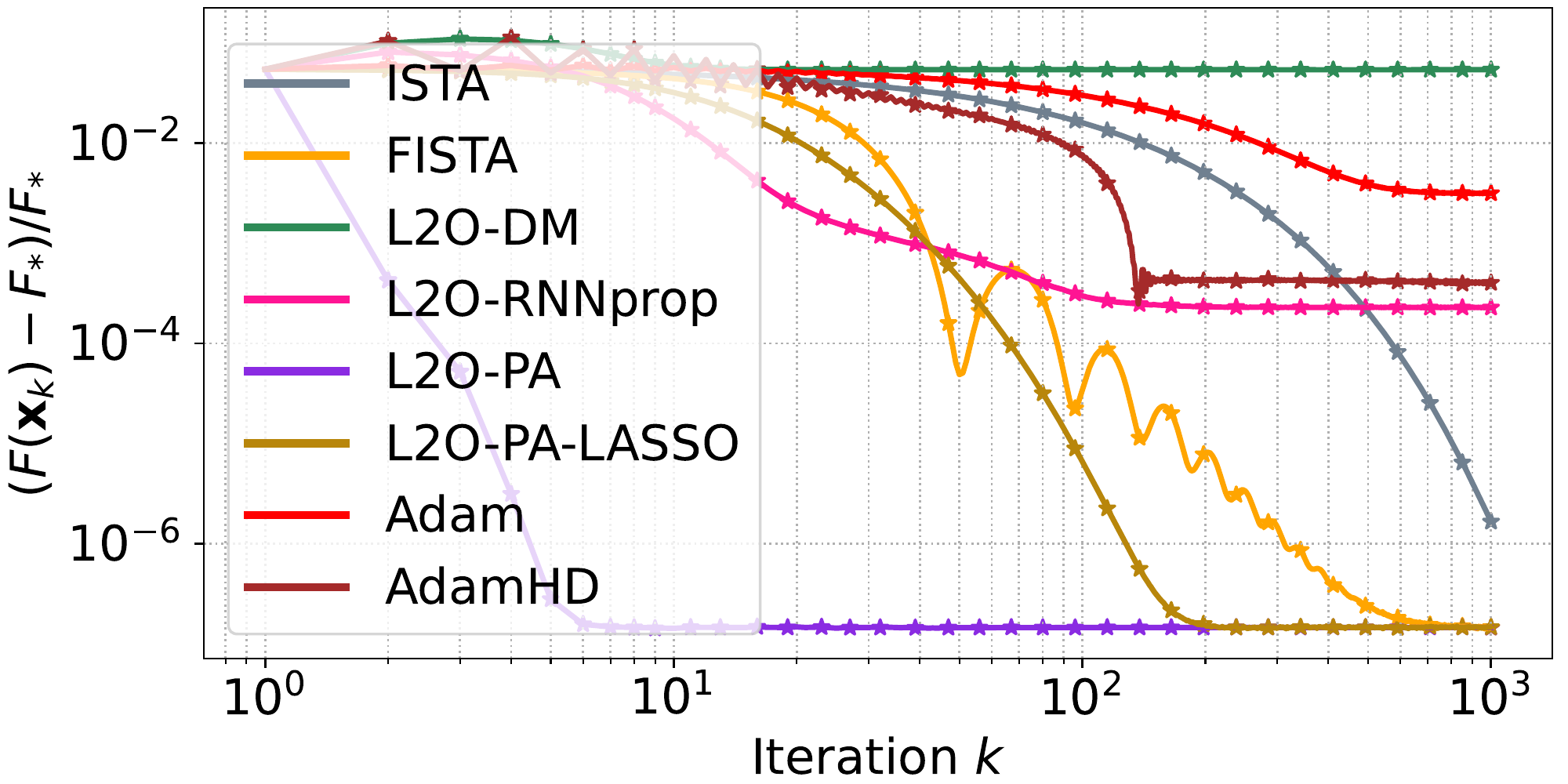}
    \vspace{-1mm}
    \caption{Logistic: Test on synthetic data.}
    \label{fig:logistic-full-ood}
\end{figure}

\paragraph{Platform.} All the experiments are conducted on a workstation equipped with four NVIDIA RTX A6000 GPUs. We used PyTorch 1.12 and CUDA 11.3.

\end{document}